\documentclass[sigplan,screen]{acmart}

\usepackage{xfp} %
\usepackage{ifthen}
\usepackage{hyperref}
\usepackage{amsmath,amsthm}
\usepackage{booktabs}   %
\usepackage[labelformat=simple]{subcaption} %

\usepackage{tikz}
\usetikzlibrary{positioning}
\usepackage{pgfplots}

\usepackage{prettyref}
\newcommand{\pref}{\prettyref}
\newrefformat{thm}{Theorem~\ref{#1}}
\newrefformat{cor}{Corollary~\ref{#1}}
\newrefformat{lem}{Lemma~\ref{#1}}
\newrefformat{cha}{Chapter~\ref{#1}}
\newrefformat{sec}{\S\!~\ref{#1}}
\newrefformat{app}{Appendix~\ref{#1}}
\newrefformat{tab}{Table~\ref{#1}}
\newrefformat{fig}{Figure~\ref{#1}}
\newrefformat{alg}{Algorithm~\ref{#1}}
\newrefformat{exa}{Example~\ref{#1}}
\newrefformat{def}{Definition~\ref{#1}}
\newrefformat{li}{Line~\ref{#1}}
\newrefformat{eq}{Equation~\ref{#1}}

\usepackage[ruled,vlined,linesnumbered]{algorithm2e}

\newcommand{\task}[1]{{\color{red}Task~#1}}

\usepackage{mathtools}
\DeclarePairedDelimiter\abs{\lvert}{\rvert}

\SetKw{returnKw}{return}
\SetKw{breakKw}{break}
\SetKw{continueKw}{continue}
\SetKwFunction{Queue}{ConstructQueue}
\SetKwFunction{Push}{Push}
\SetKwFunction{Pop}{Pop}
\SetKwFunction{Sign}{Sign}
\SetKwFunction{Next}{Next}
\SetKwFunction{Append}{Append}
\SetKwFunction{Dim}{Dim}
\SetKwFunction{CWVert}{CWVert}
\SetKwFunction{Vert}{Vert}
\SetKwFunction{Hull}{ConvexHull}
\SetKwFunction{OrthantSign}{OrthantSign}
\SetKwFunction{PlaneIntersect}{HyperplaneIntersections}
\SetKwFunction{OrthantIntersect}{EnumerateOrthantIntersections}
\SetKwFunction{SplitPlane}{SplitPlane}

\newcommand{\ProvablePatching}{{Provable Repair}}

\DeclarePairedDelimiterX{\norm}[1]{\lVert}{\rVert}{#1}
\newcommand{\RQ}[1]{\textbf{RQ#1}}
\newcommand{\subsubsubsection}[1]{\noindent\textbf{\emph{#1}}\enspace}

\newcommand\PointR{\mathtt{PointRepair}}
\newcommand\PolyR{\mathtt{PolytopeRepair}}

\newcommand\sats[2]{#1\Vdash#2}
\newcommand\satsp[2]{#1\Vdash#2}

\usepackage{ifthen}
\newcommand\target{arxiv}
\newcommand\onlyfor[3]{\ifthenelse{\equal{#1}{\target}}{#2}{#3}}

\onlyfor{pldi}{
\setcopyright{rightsretained}
\acmPrice{}
\acmDOI{10.1145/3453483.3454064}
\acmYear{2021}
\copyrightyear{2021}
\acmSubmissionID{pldi21main-p243-p}
\acmISBN{978-1-4503-8391-2/21/06}
\acmConference[PLDI '21]{Proceedings of the 42nd ACM SIGPLAN International Conference on Programming Language Design and Implementation}{June 20--25, 2021}{Virtual, Canada}
\acmBooktitle{Proceedings of the 42nd ACM SIGPLAN International Conference on Programming Language Design and Implementation (PLDI '21), June 20--25, 2021, Virtual, Canada}
}{
\setcopyright{none}
\settopmatter{printacmref=false}
\renewcommand\footnotetextcopyrightpermission[1]{}
}

\usepackage{booktabs}
\usepackage{subcaption}

\begin{document}

\title[Provable Repair of DNNs]{\ProvablePatching{} of Deep Neural Networks}

\author{Matthew Sotoudeh}
\orcid{0000-0003-2060-1009}             %
\affiliation{
  \institution{University of California, Davis}            %
  \city{Davis}
  \state{CA}
  \country{USA}                    %
}
\email{masotoudeh@ucdavis.edu}          %

\author{Aditya V.\ Thakur}
\orcid{0000-0003-3166-1517}             %
\affiliation{
  \institution{University of California, Davis}
  \city{Davis}
  \state{CA}
  \country{USA}                   %
}
\email{avthakur@ucdavis.edu}         %

\begin{abstract}
    Deep Neural Networks (DNNs) have grown in popularity over the past decade
    and are now being used in safety-critical domains such as aircraft collision
    avoidance. This has motivated a large number of techniques for finding
    unsafe behavior in DNNs. In contrast, this paper tackles the problem of
    correcting a DNN once unsafe behavior is found. We introduce the
    \emph{provable repair problem}, which is the problem of repairing a network
    $N$ to construct a new network $N'$ that satisfies a given specification. If
    the safety specification is over a finite set of points, our Provable Point
    Repair algorithm can find a provably minimal repair satisfying the
    specification, regardless of the activation functions used. For safety
    specifications addressing convex polytopes containing infinitely many
    points, our Provable Polytope Repair algorithm can find a provably minimal
    repair satisfying the specification for DNNs using piecewise-linear
    activation functions. The key insight behind both of these algorithms is the
    introduction of a \emph{Decoupled} DNN architecture, which allows us to
    reduce provable repair to a linear programming problem.
    Our experimental results demonstrate the efficiency and effectiveness of our
    Provable Repair algorithms on a variety of challenging tasks.
    \onlyfor{arxiv}{\vspace{7mm}}{}
\end{abstract}

\ccsdesc[500]{Computing methodologies~Neural networks}
\ccsdesc[500]{Theory of computation~Linear programming}
\onlyfor{arxiv}{
\ccsdesc[500]{Software and its engineering~Software post-development issues\vspace{3mm}}
}{
\ccsdesc[500]{Software and its engineering~Software post-development issues}
}

\keywords{Deep Neural Networks, Repair, Bug fixing}

\maketitle

\onlyfor{arxiv}{\newpage}{}
\section{Introduction}
\label{sec:Introduction}

Deep neural networks (DNNs)~\cite{Goodfellow:DeepLearning2016} have been
successfully applied to a wide variety of problems, including image
recognition~\cite{Krizhevsky:CACM2017}, natural-language
processing~\cite{bert2019}, medical diagnosis~\cite{KERMANY20181122}, aircraft
collision avoidance~\citep{julian2018deep}, and self-driving
cars~\citep{bojarski2016end}. However, DNNs are far from infallible, and
mistakes made by DNNs have led to loss of life \cite{teslacrash,ubercrash} and
wrongful arrests~\cite{translationarrest,wronglyaccused}. This has motivated
recent advances in understanding~\cite{Szegedy:ICLR2014,Goodfellow:ICLR2015},
verifying~\cite{Bastani:NIPS2016,reluplex:CAV2017,ai2:SP2018,Singh:POPL2019,Anderson:PLDI2019},
and
testing~\cite{pei2017deepxplore,tian2018deeptest,sun2018concolic,odena2019tensorfuzz}
of DNNs. In contrast, this paper addresses the problem of repairing a DNN once a mistake is discovered.

Consider the following motivating scenario: we have a trained
SqueezeNet~\cite{SqueezeNet}, a modern convolutional image-recognition DNN
consisting of 18 layers and 727,626 parameters. It has an accuracy of 93.6\%
on the ImageNet dataset~\cite{imagenet_cvpr09}. \emph{After deployment}, we find
that certain images are misclassified. In particular, we see that SqueezeNet has
an accuracy of only 18\% on the Natural Adversarial Examples~(NAE)
dataset~\cite{hendrycks2019nae}. \pref{fig:squirrel} shows one such image whose
actual class is \texttt{Fox Squirrel}, but the DNN predicts \texttt{Sea Lion}
with 99\% confidence. We would like to repair (or patch) the trained SqueezeNet
to ensure that it correctly classifies such images.

To repair the DNN, one could \emph{retrain} the network using the original
training dataset augmented with the newly-identified buggy inputs. Retraining,
however, is extremely inefficient; e.g., training SqueezeNet takes days or weeks
using state-of-the-art hardware. Worse, the original training dataset is often not
available for retraining; e.g., it could be private medical information, sensitive
intellectual property, or simply lost. These considerations are
more important with privacy-oriented regulations that
require companies to delete private data regularly and upon request.
Retraining can also make arbitrary changes to the DNN and, in many cases,
introduce new bugs into the DNN behavior. These issues make it infeasible,
impossible, and/or ineffective to apply retraining in many real-world
DNN-repair scenarios.

One natural alternative to retraining is \emph{fine tuning,} where we apply
gradient descent to the trained DNN but only using a smaller dataset
collected once buggy inputs are found.
While this reduces the computational cost of repair and does not require access
to the original training dataset, fine-tuning significantly increases the risk
of \emph{drawdown}, where the network forgets things it learned on the original,
larger training dataset in order to achieve high accuracy on the buggy
inputs~\cite{kemker2018measuring}. In particular, fine tuning provides no
guarantees that it makes minimal changes to the original DNN.

The effectiveness of both retraining and fine tuning to repair the DNN  is
extremely sensitive to the specific hyperparameters chosen; viz., training
algorithm, learning rate, momentum rate, etc.
Importantly, because gradient descent cannot \emph{disprove} the existence of a
better solution were some other hyperparameters picked, one might have
to try a large number of potential hyperparameter combinations in the hope
of finding one that will lead to a successful repair. This is a time-consuming
process that significantly reduces the effectiveness of such techniques in
practice.

Based on the above observations, we can deduce the following requirements for
our DNN repair algorithm, where $N$ is the buggy DNN, $X$ is the set of buggy
inputs, and $N'$ is the repaired DNN:
(P1)~\emph{efficacy}: $N'$ should correctly classify all inputs in $X$;
(P2)~\emph{generalization}: $N'$ should correctly classify inputs similar to
those in $X$;
(P3)~\emph{locality}: $N'$ should behave the same as $N$ on inputs that are
dissimilar to those in $X$;
(P4)~\emph{efficiency}: the repair algorithm should be efficient.

This paper presents a novel technique for \emph{Provable
Pointwise Repair} of DNNs that is effective, generalizing, localized, and
efficient~(\pref{sec:PointPatching}).
Given a DNN $N$ and a finite set of points $X$ along with their desired output,
our Provable Pointwise Repair algorithm synthesizes a repaired DNN $N'$ that is
guaranteed to give the correct output for all points in $X$. To ensure
locality, our algorithm can \emph{provably guarantee} that the repair
(difference in parameters) from $N$ to $N'$ is the smallest such single-layer repair. Our
Provable Pointwise Repair algorithm makes no restrictions on the activation
functions used by $N$.

\begin{figure}[t]
    \centering
    \begin{minipage}[t]{0.2\textwidth}
        \centering
        \includegraphics[height=2.5cm]{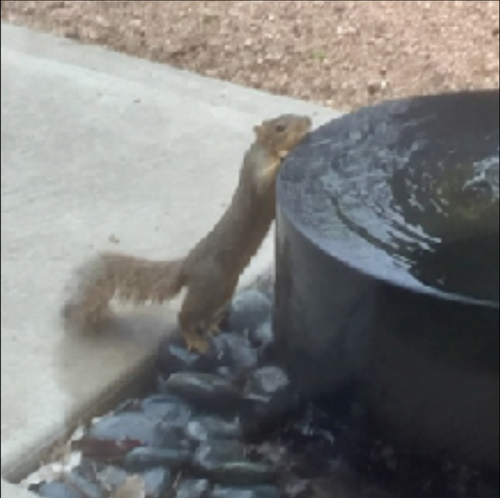}
        \caption{Natural adversarial example}
        \label{fig:squirrel}
    \end{minipage}\hfill
    \begin{minipage}[t]{0.2\textwidth}
        \centering
        \includegraphics[height=2.5cm]{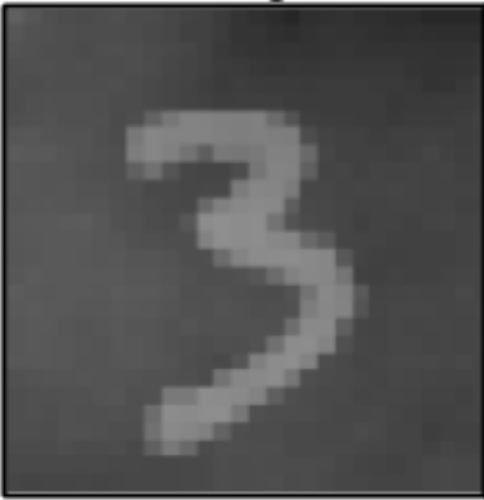}
        \caption{Fog-corrupted digit}
        \label{fig:fog3}
    \end{minipage}\hfill
    \vspace{-2.5ex}
\end{figure}

Provable repair of even a \emph{single layer} of the DNN is a
formally NP-hard problem, and completely infeasible in practice even using
state-of-the-art SMT solvers~\cite{DBLP:conf/lpar/GoldbergerKAK20}. The use of
\emph{non-linear activation functions} implies that changing even a single
weight can have a non-linear effect on the output of the DNN. However, if the
final layer of the DNN is linear (instead of a non-linear activation function),
then repairing just the output layer is actually a linear programming (LP) problem
\cite{DBLP:conf/lpar/GoldbergerKAK20} solvable in polynomial
time~\citep{khachiyanpolylp}.

The key insight to our approach is the introduction of a new DNN architecture,
called \emph{Decoupled DNNs} or DDNNs. DDNNs strictly generalize the notion of
DNNs, meaning every DNN can be trivially converted into an equivalent DDNN.
We will show that repairing \emph{any} single layer in a DDNN
reduces to an LP problem. This allows us to
compute the \emph{smallest} such single-layer repair with respect to either
the $\ell_1$ or $\ell_\infty$ norm, and, thus,
reduce forgetting.

This paper also introduces an algorithm for \emph{Provable Polytope
Repair} of DNNs (\pref{sec:PolytopePatching}), which is like Provable Point
Repair except the set of points $X$ is \emph{infinite} and specified as a union
of convex polytopes (hereafter just ``polytopes'') in the input space of the
network.

Consider a trained DNN for classifying handwritten digits, which has an accuracy
of 96.5\% on the MNIST dataset~\cite{lecun2010mnist}. After deployment, we find
that the accuracy of the network drops to 20\% on images corrupted with fog;
\pref{fig:fog3} shows an example of one such fog-corrupted image  from
MNIST-C~\cite{DBLP:journals/corr/abs-1906-02337}. We would like to repair the
network to correctly classify such fog-corrupted images. However, we might also
want to account for different amounts of fog. Let $I$ and $I_f$ be an
uncorrupted and fog-corrupted image, respectively. Then each image along the line
from $I$ to $I_f$ is corrupted by a different amount of fog. We can use
Provable Polytope Repair so that the DNN correctly classifies \emph{all
infinitely-many} such foggy images along the line from $I$ to $I_f$.

Consider an aircraft collision-avoidance network~\cite{julian2018deep}
that controls the direction an aircraft should turn based on the relative
position of an attacking aircraft. We may want this DNN to satisfy certain
properties, such as never instructing the aircraft to turn towards the attacker
when the attacker is within a certain distance. Our Provable Polytope Repair
algorithm can synthesize a repaired DNN that provably satisfies such safety
properties on an infinite set of input scenarios.

The main insight for solving Provable Polytope Repair is that, for
piecewise-linear DDNNs, repairing polytopes (with infinitely many points) is
equivalent to Provable Point Repair on finitely-many \emph{key points}.
These key points can be computed for DDNNs using prior work
on computing symbolic representations of DNNs~\cite{exactline,syrenn}. This reduction
is intuitively similar to how the simplex algorithm reduces optimizing over a
polytope with infinitely many points to optimizing over the finitely-many
vertex points. As illustrated by the above two scenarios, there are practical
applications in which the polytopes used in the repair specification are
low-dimensional subspaces of the input space of the DNNs.

We evaluate the efficiency and efficacy of our Provable Repair
algorithms compared to fine tuning (\pref{sec:ExperimentalEvaluation}).
The repairs by our algorithms \emph{generalize} to similarly-buggy inputs while
avoiding significant \emph{drawdown}, or forgetting.

The contributions of the paper are:
\begin{itemize}
    \item We introduce Decoupled DNNs, a new DNN architecture that enables
    efficient and effective repair~(\pref{sec:DDNNs}).
    \item An algorithm for Provable Point Repair~(\pref{sec:PointPatching}).
    \item An algorithm for Provable Polytope Repair of piecewise-linear
    DNNs~(\pref{sec:PolytopePatching}).
    \item Experimental evaluation of Provable
        Repair~(\pref{sec:ExperimentalEvaluation}).
\end{itemize}
\pref{sec:Preliminaries} describes preliminaries;
\pref{sec:Overview} presents an overview of our approach;
\pref{sec:RelatedWork} describes related work; \pref{sec:Conclusion} concludes.

\section{Preliminaries}
\label{sec:Preliminaries}

A feed-forward DNN is a special type of loop-free computer program that
computes a vector-valued function. DNNs are often represented as layered DAGs.
An input to the network is given by associating with each node in the input
layer one component of the input vector. Then each node in the second layer
computes a weighted sum of the nodes in the input layer according to the edge
weights. The output of each node in the second layer is the image of this
weighted sum under some \emph{non-linear activation function} associated with
the layer. This process is repeated until output values at the final layer are
computed, which form the components of the output vector.

Although we will use the above DAG definition of a DNN for the intuitive examples
in~\pref{sec:Overview}, for most of our formal theorems we will use an entirely
equivalent definition of DNNs, below, as an alternating concatenation of \emph{linear}
and \emph{non-linear} functions.

\begin{definition}
    \label{def:DNN}
    A \emph{Deep Neural Network} (DNN) with layer sizes $s_0, s_1, \ldots, s_n$
    is a list of tuples $(W^{(1)}, \sigma^{(1)}), \ldots, (W^{(n)},
    \sigma^{(n)})$, where each $W^{(i)}$ is an $s_i \times s_{i-1}$ matrix and
    $\sigma^{(i)} : \mathbb{R}^{s_i}\to \mathbb{R}^{s_i}$ is some
    \emph{activation function}.
\end{definition}
\begin{definition}
    \label{def:DNNFunction}
    Given a DNN $N$ with layers $(W^{(i)}, \sigma^{(i)})$ we say the
    \emph{function associated with the DNN} is a function $N: \mathbb{R}^{s_0}
    \to \mathbb{R}^{s_n}$ given by
    $
        N(\vec{v}) = \vec{v}^{(n)}
    $
    where
    $  \vec{v}^{(0)} \coloneqq \vec{v}$
        and
        $\vec{v}^{(i)} \coloneqq \sigma^{(i)}(W^{(i)}\vec{v}^{(i-1)}).$
\end{definition}

For ease of exposition, we have assumed: (i) that every layer in the DNN is
fully-connected, i.e., parameterized by an entire weight matrix, and (ii) that
the activation functions $\sigma^{(i)}$ have the same domain and range.
However, our algorithms do not rely on these conditions, and in fact, we use
more complicated DNNs (such as Convolutional Neural Networks) in our
evaluations (\pref{sec:ExperimentalEvaluation}).

There are a variety of activation functions used for $\sigma^{(i)}$, including
ReLU, Hyperbolic Tangent, (logistic) Sigmoid, AveragePool,
and MaxPool~\citep{Goodfellow:DeepLearning2016}. In our examples, we will use
the ReLU function, defined below, due to its simplicity and use in real-world
DNN architectures, although our algorithms and theory work for arbitrary
activation functions.

\begin{definition}
    The \emph{ReLU Activation Function} is a vector-valued function
    $\mathbb{R}^n\to \mathbb{R}^n$ defined component-wise by
    \[
        ReLU(\vec{v})_i =
        \begin{cases}
            v_i &\text{if } v_i \geq 0 \\
            0 &\text{otherwise,}
        \end{cases}
    \]
    where $ReLU(\vec{v})_i$ is the $i^{\mathit{th}}$ component of the output vector
    $ReLU(\vec{v})$ and $v_i$ is the $i^{\mathit{th}}$ of the input vector
    $\vec{v}$.
\end{definition}

Of particular note for our polytope repair algorithm
(\pref{sec:PolytopePatching}), some of the most common activation functions
(particularly ReLU) are \emph{piecewise-linear}.
\begin{definition}
    \label{def:PWL}
    A function $f : \mathbb{R}^n \to \mathbb{R}^m$ is \emph{piecewise-linear}
    (PWL) if its input domain can be partitioned into finitely-many
    polytopes $X_1, X_2, \ldots, X_n$ such that, for each $X_i$, there exists
    some \emph{affine} function $f_i$ such that $f(x) = f_i(x)$ for every $x
    \in X_i$.
\end{definition}

This paper uses the terms `linear' and `affine'
interchangeably. It follows from \pref{def:PWL} that
compositions of PWL functions are themselves PWL. Hence, a DNN using only PWL
activation functions is also PWL in its entirety.

For a network using PWL activation functions, we can always associate with each
input to the network an \emph{activation pattern}, as defined below.
\begin{definition}
    Let $N$ be a DNN using only PWL activation functions. Then an
    \emph{activation pattern} $\gamma$ is a mapping from each activation
    function $\sigma^{(j)}$ to a linear region $\gamma(\sigma^{(j)})$ of
    $\sigma^{(j)}$. We say an activation pattern $\gamma$ \emph{holds} for a vector
    $\vec{v}$ if, for every layer $j$, we have $W^{(j)}\vec{v}^{(j-1)} \in
    \gamma(\sigma^{(j)})$.
\end{definition}
Recall that $\vec{v}$ is the input to the first layer of the network while
$\vec{v}^{(j - 1)}$ is the intermediate input to the $j$th layer. For example,
suppose $N$ is a ReLU network where $\gamma$ holds for vector
$v$, then $\gamma(\sigma^{(j)})$ is exactly the set of nodes in layer $j$ with
positive output when evaluating the DNN on input $\vec{v}$.

Let $N$ be a DNN that uses only PWL activation functions.
Then we notate by $LinRegions(N)$ the set of polytopes $X_1, X_2, \ldots, X_n$
that partition the domain of $N$ such that the conditions in~\pref{def:PWL}
hold.
In particular, we will use the partitioning for which we can assign each
$X_i$ a unique activation pattern $\gamma_i$ such that $\gamma_i$ holds for all
$\vec{v} \in X_i$.

When appropriate, for polytope $P$ in the domain of $N$, we will notate by
$LinRegions(N, P)$ a partitioning $X_1, X_2, \ldots, X_n$ \emph{of $P$} that
meets the conditions in~\pref{def:PWL}. Formally, we have $LinRegions(N, P)
\coloneqq LinRegions(N_{\restriction P})$, where $N_{\restriction P}$ is the
restriction of $N$ to domain $P$.

Consider the ReLU DNN $N_1$ shown in~\pref{fig:OverviewDNN}, which has one
input $x$, one output $y$, and three so-called \emph{hidden} nodes $h_1$,
$h_2$, and $h_3$ using ReLU activation function.  We will consider the
input-output behavior of this network for the domain $x \in [-1, 2]$.  The
linear regions of $N_1$ are shown visually in~\pref{fig:OverviewDNNPlot} as
colored intervals on the $x$ axis, which each map into the postimage according
to some affine mapping which is specific to that region.  In particular, we
have three linear regions:
\begin{equation}
    \label{eq:OverviewLinRegionsN1}
    LinRegions(N_1, [-1, 2]) = \{ [-1, 0], [0, 1], [1, 2] \}.
\end{equation}
Each linear region corresponds to a particular \emph{activation pattern} on the
hidden nodes; i.e., which ones are in the zero region or the identity region.
The first linear region, $[-1, 0]$ ({\color{red}red}), corresponds to the
activation pattern where only $h_1$ is activated. The second linear region,
$[0, 1]$ ({\color{blue}blue}), corresponds to the activation pattern where only
$h_2$ is activated. Finally, the third linear region, $[1, 2]$
({\color{green}green}), corresponds to the activation pattern where both $h_2$
and $h_3$ are activated.

In practice, we can quickly compute $LinRegions(N, P)$ for either large $N$
with one-dimensional $P$ or medium-sized $N$ with two-dimensional $P$.
We use the algorithm of~\citet{syrenn}
for computing $LinRegions(N, P)$ when $P$ is one- or two-dimensional.

\begin{definition}
    \label{def:LP}
    A \emph{linear program} (LP) with $n$ constraints on $m$ variables is a
    triple $(A, \vec{b}, \vec{c})$ where $A$ is an $n\times m$ matrix,
    $\vec{b}$ is an $n$-dimensional vector, and $\vec{c}$ is an $m$-dimensional
    vector.

    A \emph{solution to the linear program} is an $m$-dimensional vector
    $\vec{x}$ such that (i) $A\vec{x} \leq \vec{b}$, and (ii) $\vec{c}\cdot
    \vec{x}$ is minimal among all $\vec{x}$ satisfying (i).
\end{definition}

Linear programs can be solved in polynomial time~\citep{khachiyanpolylp}, and
many efficient, industrial-grade LP solvers such as the Gurobi solver~\citep{gurobi}
exist.  Through the addition of auxiliary variables, it is also possible to
encode in an LP the objective of minimizing the $\ell_1$ and/or $\ell_\infty$
norms of $\vec{x}$~\citep{abslp}.

\section{Overview}
\label{sec:Overview}

This paper discusses how to repair DNNs to enforce precise
\emph{specifications}, i.e., constraints on input-output behavior.

\newcommand{\overviewdnn}[1]{
    \tikzstyle{dnnnode} = [circle, draw, inner sep=0pt,minimum size=3.5ex]
    \draw node[dnnnode] (x) at (0, 0) {$x$};
    \draw node[dnnnode] (b) at (0, 1) {$1$};
    \draw node[dnnnode] (h1) at (1.5, -1) {$h_1$};
    \draw node[dnnnode] (h2) at (1.5, 0) {$h_2$};
    \draw node[dnnnode] (h3) at (1.5, 1) {$h_3$};
    \draw node[dnnnode] (y) at (3, 0) {$y$};

    \draw[->] (b) -- (h3) node[midway,above] {$-1$};
    \draw[->] (x) -- (h1) node[midway,below] {$-1$};
    \draw[->] (x) -- (h2) node[midway,above] {$1$};
    \draw[->] (x) -- (h3) node[midway,above] {#1};
    \draw[->] (h1) -- (y) node[midway,below] {$-1$};
    \draw[->] (h2) -- (y) node[midway,above] {$-1$};
    \draw[->] (h3) -- (y) node[midway,above] {$1$};
}

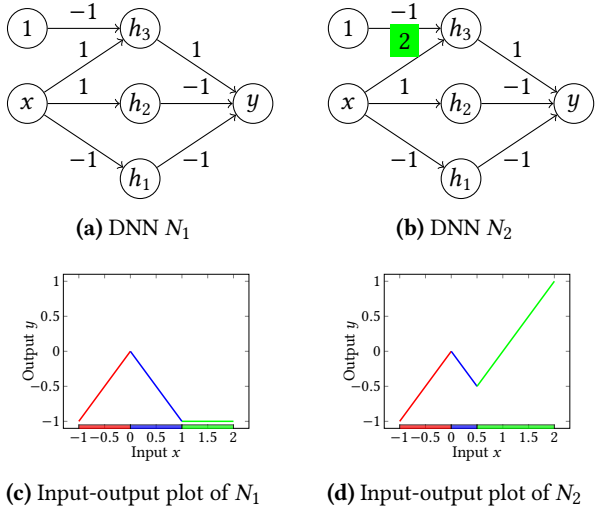
\begin{figure}[t]
    \begin{subfigure}[t]{0.4\linewidth}
        \centering
        \begin{tikzpicture}
            \overviewdnn{$1$}
        \end{tikzpicture}
        \caption{DNN $N_1$}
        \label{fig:OverviewDNN}
    \end{subfigure}
    \hspace{2em}
    \begin{subfigure}[t]{0.4\linewidth}
        \centering
        \begin{tikzpicture}
            \overviewdnn{\colorbox{green}{$2$}}
        \end{tikzpicture}
        \caption{DNN $N_2$}
        \label{fig:OverviewModifiedDNN}
    \end{subfigure}
    \\
    \vspace{1em}
    \begin{subfigure}[t]{0.4\linewidth}
        \centering
        \begin{tikzpicture}[scale=0.36]
            \begin{axis}[ymin=-1.1,ymax=1.1,xlabel={Input $x$},ylabel={Output $y$},font=\huge]
                \addplot[ultra thick,red,domain=-1:0,samples=2] {x};
                \addplot[ultra thick,blue,domain=0:1,samples=2] {-x};
                \addplot[ultra thick,green,domain=1:2,samples=2] {-x + (x - 1)};
                \draw[fill=red,opacity=.7] (axis cs:-1,-2) rectangle (axis cs:0,-1.05);
                \draw[fill=blue,opacity=.7] (axis cs:0,-2) rectangle (axis cs:1,-1.05);
                \draw[fill=green,opacity=.7] (axis cs:1,-2) rectangle (axis cs:2,-1.05);
            \end{axis}
        \end{tikzpicture}
        \caption{Input-output plot of $N_1$}
        \label{fig:OverviewDNNPlot}
    \end{subfigure}
    \hspace{2em}
    \begin{subfigure}[t]{0.4\linewidth}
        \centering
        \begin{tikzpicture}[scale=0.36]
            \begin{axis}[ymin=-1.1,ymax=1.1,xlabel={Input $x$},ylabel={Output $y$},font=\huge]
                \addplot[ultra thick,red,domain=-1:0,samples=2] {x};
                \addplot[ultra thick,blue,domain=0:0.5,samples=2] {-x};
                \addplot[ultra thick,green,domain=0.5:2,samples=2] {-x + ((2*x) - 1)};
                \draw[fill=red,opacity=.7] (axis cs:-1,-2) rectangle (axis cs:0,-1.05);
                \draw[fill=blue,opacity=.7] (axis cs:0,-2) rectangle (axis cs:0.5,-1.05);
                \draw[fill=green,opacity=.7] (axis cs:0.5,-2) rectangle (axis cs:2,-1.05);
            \end{axis}
        \end{tikzpicture}
        \caption{Input-output plot of $N_2$}
        \label{fig:OverviewModifiedDNNPlot}
    \end{subfigure}
    \caption{
        Example DNNs and their input-output behavior.
        The $h_i$ nodes have ReLU activation functions.
        Colored bars on the $x$ axis denote the linear regions.}
\end{figure}

\subsection{Provable Pointwise Repair}
\label{sec:OverviewPointPatching}
The first type of
specification we will consider is a \emph{point repair specification.} In
this scenario, we are given a finite set of input points along
with, for each such point, a subset of the output region which we would like
that point to be mapped into by the network.

Consider DNN $N_1$ in \pref{fig:OverviewDNN}.
We see that $N_1(0.5) = -0.5$ and $N_1(1.5) = -1$. We
want to \emph{repair} it to form a new network $N'$ such that
\begin{equation}
    \label{eq:OverviewPointSpec}
    (-1 \leq N'(0.5) \leq -0.8) \wedge (-0.2 \leq N'(1.5) \leq 0).
\end{equation}
We formalize this point specification as $(X, A^\cdot, b^\cdot)$ where $X$ is a
finite collection of \emph{repair points}
$X = \{ X_1 = 0.5, X_2 = 1.5 \}$,
and we associate with each $x \in X$ a polytope in the output space
defined by $A^x, b^x$ that we would like it to be mapped into by $N'$. In this
case, we can let
$
    A^{X_1} =
    {\footnotesize \begin{bmatrix}
        1 \\ -1
    \end{bmatrix}}
$,$
    b^{X_1} =
    {\footnotesize \begin{bmatrix}
        -0.8 \\ 1
    \end{bmatrix}}$,$
    A^{X_2} =
    {\footnotesize \begin{bmatrix}
        1 \\ -1
    \end{bmatrix}}$, and $
    b^{X_2} =
    {\footnotesize \begin{bmatrix}
        0 \\ 0.2
    \end{bmatrix}}
$
representing the polyhedral constraints
$
    A^{X_1} N'(X_1) \leq b^{X_1} \wedge A^{X_2} N'(X_2) \leq b^{X_2}.
$
These constraints are equivalent
to~\pref{eq:OverviewPointSpec}.

The general affine constraint
form we use is very expressive. For example, it can express constraints such as
``the $i^{\mathrm{th}}$ output component is larger than all others,'' which for a
multi-label classification network is equivalent to ensuring that the point is
classified with label $i$.

\subsubsubsection{The two roles of a ReLU.}
At first glance, it is tempting to directly encode the DNN in an SMT solver
like Z3~\citep{TACAS:deMB08} and attempt to solve for weight assignments that
cause the desired classification. However, in practice this quickly becomes
infeasible even for networks with very few nodes.

To understand the key reason for this infeasibility, consider what
happens when a single weight in $N_1$ is modified to construct the new DNN
$N_2$, shown in~\pref{fig:OverviewModifiedDNN}. In particular,
the weight on $x\to h_3$ is changed from a 1 to a 2.
Comparing~\pref{fig:OverviewModifiedDNNPlot} with~\pref{fig:OverviewDNNPlot},
we see that changing this weight has caused \emph{two} distinct changes in the
plot:
\begin{enumerate}
    \item The linear function associated with the green region has changed,
        and
    \item \emph{Simultaneously,} the linear regions themselves (shown on the
        $x$ axis) have changed, with the green region growing to include parts
        of the space originally in the blue region. In particular,
        $LinRegions(N_2, [-1, 2]) = \{[-1, 0], [0, 0.5], [0.5, 2]\}$, different
        from~\pref{eq:OverviewLinRegionsN1}.
\end{enumerate}

We use \emph{coupling} to refer to the fact that the weights in a ReLU
DNN simultaneously control \emph{both} of these aspects. This coupling causes
repair of DNNs to be computationally infeasible, because the
impact of changing a weight in the network with respect to the output of the
network on a fixed input is non-linear; it `jumps' every time the linear region
that the point falls into changes.
This paper shows that \emph{de-}coupling these two roles leads to a generalized
class of neural networks along with a polynomial-time repair algorithm.

\newcommand\avnodes[6][]{
    \tikzstyle{dnnnode} = [circle, inner sep=0pt,minimum size=3.5ex]
    \draw node[dnnnode,draw=red] (a#2) at (#3, #4) {#5};
    \draw node[dnnnode,draw=black] (v#2) at (\fpeval{#3 + 2.5}, \fpeval{#4 - 2.5}) {#6};
    \ifthenelse{\equal{#1}{mask}}{
        \draw [-|,blue,line width=0.5mm,opacity=0.6,dashed] (a#2) -- (v#2);
    }{}
}
\newcommand\avconnect[5][midway,above]{
    \draw[->] (a#2) -- (a#3) node[#1] {#4};
    \draw[->] (v#2) -- (v#3) node[#1] {#5};
}
\newcommand{\overviewddnnnodes}{
    \avnodes{x}{0}{0}{$x^a$}{$x^v$}
    \avnodes{b}{0}{1}{$1$}{$1$}
    \avnodes[mask]{y1}{1.5}{-1}{$h_1^a$}{$h_1^v$}
    \avnodes[mask]{y2}{1.5}{0}{$h_2^a$}{$h_2^v$}
    \avnodes[mask]{y3}{1.5}{1}{$h_3^a$}{$h_3^v$}
    \avnodes{z}{3}{0}{$y^a$}{$y^v$}
}
\newcommand{\overviewddnnedges}[2]{
    \avconnect{b}{y3}{$-1$}{$-1$}
    \avconnect[midway,below]{x}{y1}{$-1$}{$-1$}
    \avconnect{x}{y2}{$1$}{$1$}
    \ifthenelse{\equal{#2}{$1$}}{
        \avconnect[near start,above]{x}{y3}{#1}{\colorbox{white}{#2}}
    }{
        \avconnect[near start,above]{x}{y3}{#1}{\colorbox{green}{#2}}
    }
    \avconnect[near end,below]{y1}{z}{$-1$}{$-1$}
    \avconnect[yshift=-0.75mm,near start,above]{y2}{z}{$-1$}{$-1$}
    \avconnect{y3}{z}{$1$}{$1$}
}
\begin{figure}[t]
    \hspace{-2em}
    \begin{subfigure}[t]{0.4\linewidth}
        \begin{tikzpicture}[scale=0.75]
            \small
            \overviewddnnnodes{}
            \overviewddnnedges{$1$}{$1$}
        \end{tikzpicture}
        \caption{\footnotesize{Decoupled DNN $N_3 = (N_1, N_1)$.}}
        \label{fig:DecoupledDNN}
    \end{subfigure}
    \hspace{2em}
    \begin{subfigure}[t]{0.4\linewidth}
        \centering
        \begin{tikzpicture}[scale=0.75]
            \small
            \overviewddnnnodes
            \overviewddnnedges{$1$}{$2$}
        \end{tikzpicture}
        \caption{\footnotesize{Decoupled DNN $N_4 = (N_1, N_2)$.}}
        \label{fig:DecoupledDNNValue}
    \end{subfigure}
    \\
    \vspace{1em}
    \begin{subfigure}[t]{0.4\linewidth}
        \centering
        \begin{tikzpicture}[scale=0.36]
            \begin{axis}[ymin=-1.1,ymax=1.1,xlabel={Input $x^v (= x^a)$},ylabel={Output $y^v$},font=\huge]
                \addplot[ultra thick,red,domain=-1:0,samples=2] {x};
                \addplot[ultra thick,blue,domain=0:1,samples=2] {-x};
                \addplot[ultra thick,green,domain=1:2,samples=2] {-x + (x - 1)};
                \draw[fill=red,opacity=.7] (axis cs:-1,-2) rectangle (axis cs:0,-1.05);
                \draw[fill=blue,opacity=.7] (axis cs:0,-2) rectangle (axis cs:1,-1.05);
                \draw[fill=green,opacity=.7] (axis cs:1,-2) rectangle (axis cs:2,-1.05);
            \end{axis}
        \end{tikzpicture}
        \caption{\footnotesize{Input-output plot of $N_3$.}}
        \label{fig:DecoupledDNNPlot}
    \end{subfigure}
    \hspace{2em}
    \begin{subfigure}[t]{0.4\linewidth}
        \centering
        \begin{tikzpicture}[scale=0.36]
            \begin{axis}[ymin=-1.1,ymax=1.1,xlabel={Input $x^v (= x^a)$},ylabel={Output $y^v$},font=\huge]
                \addplot[ultra thick,red,domain=-1:0,samples=2] {x};
                \addplot[ultra thick,blue,domain=0:1,samples=2] {-x};
                \addplot[ultra thick,green,domain=1:2,samples=2] {-x + ((2*x) - 1)};
                \draw[fill=red,opacity=.7] (axis cs:-1,-2) rectangle (axis cs:0,-1.05);
                \draw[fill=blue,opacity=.7] (axis cs:0,-2) rectangle (axis cs:1,-1.05);
                \draw[fill=green,opacity=.7] (axis cs:1,-2) rectangle (axis cs:2,-1.05);
            \end{axis}
        \end{tikzpicture}
        \caption{\footnotesize{Input-output plot of $N_4$.}}
        \label{fig:DecoupledDNNValuePlot}
      \end{subfigure}
    \caption{Decoupled DNNs $N_3$ and $N_4$ and their input-output behavior.
    DDNNs $N_3$ and $N_4$ have the same activation channel $N_1$, but different value channels.}
\end{figure}
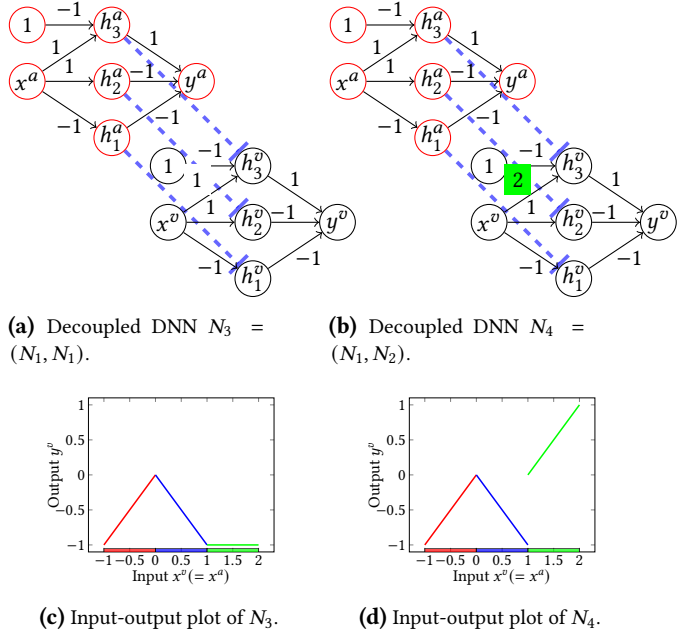

\subsubsubsection{Decoupling activations from values.}
The key insight of this paper is a novel DNN architecture, \emph{Decoupled
DNNs} (DDNNs), defined in~\pref{sec:DDNNs}, that strictly generalizes standard
feed-forward DNNs while at the same time allowing us to decouple the two roles
that the parameters play.

\pref{fig:DecoupledDNN} shows a DDNN $N_3$ equivalent to $N_1$
from \pref{fig:OverviewDNN}. Most notably, every decoupled DNN consists of
\emph{two} `sub-networks,' or \emph{channels.} The \emph{activation channel,}
shown in red, is used to determine the positions of the linear regions.
Meanwhile, the \emph{value channel} determines the output map within
each linear region. The activation channel influences the value channel via the
blue edges, which indicate that the adjacent value node is activated only if
the corresponding activation node is. For example, if the input to $h_2^a$ is
negative, then $h_2^v$ will output zero regardless of the input to~$h_2^v$.

To compute the output of a DDNN on a given input $x_0$, we first set $x^a =
x_0$, evaluate the activation channel, and record which of the hidden nodes
$h^a_i$ were active (received a positive input) or inactive (otherwise). Then,
we set $x^v = x_0$ and evaluate the value channel, except instead of activating
a node if \emph{its} input is non-negative, we activate the node if \emph{the
corresponding activation channel node was activated.} In this way, activation
nodes can `mask' their corresponding value nodes, as notated with the blue
edges in~\pref{fig:DecoupledDNN}.

Now, consider what happens when we change a weight in only the \emph{value
channel,} as shown in~\pref{fig:DecoupledDNNValue}. In that scenario, on any
given point, the activation pattern for any given input \emph{does not change,}
and so the locations of the linear regions on the $x$ axis
of~\pref{fig:DecoupledDNNValuePlot} are unchanged
from~\pref{fig:DecoupledDNNPlot}. However, what we find is that \emph{the
linear function within any given region does change.} Note that in this case
only the green line has changed, however in deeper networks changing any given
weight can change all of the lines.

\subsubsubsection{Repair of DDNNs.}
This observation foreshadows two of our key theoretical results
in~\pref{sec:DDNNs}. The first theorem (\pref{thm:LinearDDNN}) shows that, for any
given input, the output of the DDNN varies \emph{linearly} on the change of
any given weight in the value channel. In fact, we will show the stronger fact
that the output depends linearly with the change of any \emph{layer of weights}
in the value channel.

Using this fact, we can \emph{reduce pointwise repair of a single layer in the
DDNN to a linear programming (LP) problem.} In the running example, suppose we
want to repair the first value layer of DDNN $N_3$ to
satisfy~\pref{eq:OverviewPointSpec}. Let $\Delta$ be the difference in the
first layer weights, where $\Delta_i$ is the change in the weight on edge $x^v \to h^v_i$,
$\Delta_4$ is the change in the weight on edge $1 \to h^v_3$, and $N'$ be the DDNN with
first-layer value weights changed by $\Delta$. Then, \pref{thm:LinearDDNN} guarantees
that
$
    \small
    N'(X_1) =
    \begin{bmatrix} -0.5 \end{bmatrix}
    +
    \begin{bmatrix} 0 & -0.5 & 0 & 0 \end{bmatrix}
    \vec{\Delta}
    = -0.5 - 0.5\Delta_2,
$\\
while
$
    \small
    N'(X_2) =
    \begin{bmatrix} -1 \end{bmatrix}
    +
    \begin{bmatrix} 0 & -1.5 & 1.5 & 1 \end{bmatrix}
    \vec{\Delta}
    = -1 - 1.5\Delta_2 + 1.5\Delta_3 + \Delta_4.
$
Hence, we can encode our specification as an LP like so:
$\small
    (-1 \leq -0.5 - 0.5\Delta_2 \leq -0.8)
    \wedge
    (-0.2 \leq -1 - 1.5\Delta_2 + 1.5\Delta_3 + \Delta_4 \leq 0),
$
or in a more formal LP form,
\[
    \footnotesize
    \begin{bmatrix}
        0 & -0.5 & 0 & 0 \\
        0 & 0.5 & 0 & 0 \\
        0 & -1.5 & 1.5 & 1 \\
        0 & 1.5 & -1.5 & -1 \\
    \end{bmatrix}
    \begin{bmatrix}
        \Delta_1 \\ \Delta_2 \\ \Delta_3 \\ \Delta_4
    \end{bmatrix}
    \leq
    \begin{bmatrix}
        -0.3 \\ 0.5 \\ 1 \\ -0.8 \\
    \end{bmatrix}
\]

We can then solve for $\Delta$ using an off-the-shelf LP solver, such as
Gurobi~\citep{gurobi}.  We can also simultaneously optimize a linear
objective, such as the $\ell_\infty$ or $\ell_1$ norm, to find the
satisfying repair with the \emph{provably} smallest $\Delta$. This helps ensure
locality of the repair and preserve the otherwise-correct existing behavior of
the network. In this case, we can find that the smallest repair with respect to
the $\ell_1$ norm is
$
    \Delta_1 = 0, \Delta_2 = 0.6, \Delta_3 = 1.1\overline{3}, \Delta_4 = 0.
$
The corresponding repaired DDNN $N_5$ is shown
in~\pref{fig:OverviewPointPatchedDDNN} and plotted
in~\pref{fig:OverviewPointPatchedDDNNPlot}, where we can see that the repaired
network satisfies the constraints because $N_5(0.5) = -0.8$ and $N_5(1.5) =
-0.2$.  Notably, the linear regions of $N_5$ are the same as those of $N_1$.

\subsubsubsection{Non-ReLU, non-fully-connected, activation functions.}
While we have focused in this overview on the ReLU case for ease of
exposition, the key result of~\pref{thm:LinearDDNN} also holds for a
generalization of DDNNs using arbitrary activation functions, such as $\tanh$
and sigmoid. Hence, our pointwise repair algorithm works for arbitrary
feed-forward networks. Similarly, although we have formalized DNNs assuming
fully-connected layers, our approach can repair convolutional and other similar
types of layers as well (as demonstrated in~\pref{sec:EvaluationImageNet}).

\subsection{Provable Polytope Repair}
\label{sec:OverviewPolytopePatching}

We now consider \emph{Provable Polytope Repair}. The specification for provable
polytope repair constrains the output of the network on finitely-many
\emph{polytopes} in the input space, each one containing potentially
\emph{infinitely many points.} For example, given the DNN $N_1$ we may wish to
enforce a specification
\begin{equation}
    \label{eq:OverviewPolytopeSpec}
    \forall x \in [0.5, 1.5]. \quad -0.8 \leq N'(x) \leq -0.4.
\end{equation}
We represent this as a polytope specification with one input polytope,
$X = \{ P_1 = [0.5, 1.5] \}$, which should map to the polytope in the output
space given by
$A^{P_1} = {\footnotesize\begin{bmatrix} 1 \\ -1 \end{bmatrix}}$,
$b^{P_1} = {\footnotesize\begin{bmatrix} -0.4 \\ 0.8 \end{bmatrix}}$.
The constraint $\forall x \in P_1. A^{P_1}N'(x) \leq
b^{P_1}$ is then equivalent to the specification in~\pref{eq:OverviewPolytopeSpec}.

\begin{figure}[t]
    \hspace{-2em}
    \begin{subfigure}[t]{0.4\linewidth}
        \begin{tikzpicture}[scale=0.75]
            \small
            \overviewddnnnodes{}
            \avconnect{b}{y3}{$-1$}{$-1$}
            \avconnect[midway,below]{x}{y1}{$-1$}{$-1$}
            \avconnect[near end,above,yshift=-0.75mm]{x}{y2}{$1$}{\tiny\colorbox{green}{$1.6$}}
            \avconnect[near start,above,yshift=-0.3mm]{x}{y3}{$1$}{\tiny\colorbox{green}{$2.1\overline{3}$}}
            \avconnect[near end,below]{y1}{z}{$-1$}{$-1$}
            \avconnect[yshift=-0.75mm,near start,above]{y2}{z}{$-1$}{$-1$}
            \avconnect{y3}{z}{$1$}{$1$}
        \end{tikzpicture}
        \caption{Pointwise Repaired DDNN $N_5$.}
        \label{fig:OverviewPointPatchedDDNN}
    \end{subfigure}
    \hspace{2em}
    \begin{subfigure}[t]{0.4\linewidth}
        \begin{tikzpicture}[scale=0.75]
            \small
            \overviewddnnnodes{}
            \avconnect{b}{y3}{$-1$}{$-1$}
            \avconnect[midway,below]{x}{y1}{$-1$}{$-1$}
            \avconnect[near end,above,yshift=-0.75mm]{x}{y2}{$1$}{\tiny\colorbox{green}{$0.8$}}
            \avconnect[near start,above]{x}{y3}{$1$}{$1$}
            \avconnect[near end,below]{y1}{z}{$-1$}{$-1$}
            \avconnect[yshift=-0.75mm,near start,above]{y2}{z}{$-1$}{$-1$}
            \avconnect{y3}{z}{$1$}{$1$}
        \end{tikzpicture}
        \caption{Polytope Repaired DDNN $N_6$.}
        \label{fig:OverviewPolytopePatchedDDNN}
    \end{subfigure} \\
    \vspace{1em}
    \begin{subfigure}[t]{0.4\linewidth}
        \centering
        \begin{tikzpicture}[scale=0.36]
            \begin{axis}
            [ymin=-1.1,ymax=1.1,xlabel={Input $x^v (= x^a)$},ylabel={Output $y^v$},font=\huge,
             xmin=-1.15,xmax=2.15]
                \addplot[ultra thick,red,domain=-1:0,samples=2] {x};
                \addplot[ultra thick,blue,domain=0:1,samples=2] {-(1.6*x)};
                \addplot[ultra thick,green,domain=1:2,samples=2] {-(1.6*x) + (2.13333*x - 1)};
                \draw[fill=red,opacity=.7] (axis cs:-1,-2) rectangle (axis cs:0,-1.05);
                \draw[fill=blue,opacity=.7] (axis cs:0,-2) rectangle (axis cs:1,-1.05);
                \draw[fill=green,opacity=.7] (axis cs:1,-2) rectangle (axis cs:2,-1.05);
            \end{axis}
        \end{tikzpicture}
        \caption{Input-output plot of $N_5$.}
        \label{fig:OverviewPointPatchedDDNNPlot}
    \end{subfigure}
    \hspace{2em}
    \begin{subfigure}[t]{0.4\linewidth}
        \centering
        \begin{tikzpicture}[scale=0.36]
            \begin{axis}
            [ymin=-1.1,ymax=1.1,xlabel={Input $x^v (= x^a)$},ylabel={Output $y^v$},font=\huge,
             xmin=-1.15,xmax=2.15]
                \addplot[ultra thick,red,domain=-1:0,samples=2] {x};
                \addplot[ultra thick,blue,domain=0:1,samples=2] {-(0.8*x)};
                \addplot[ultra thick,green,domain=1:2,samples=2] {-(0.8*x) + (x - 1)};
                \draw[fill=red,opacity=.7] (axis cs:-1,-2) rectangle (axis cs:0,-1.05);
                \draw[fill=blue,opacity=.7] (axis cs:0,-2) rectangle (axis cs:1,-1.05);
                \draw[fill=green,opacity=.7] (axis cs:1,-2) rectangle (axis cs:2,-1.05);
            \end{axis}
        \end{tikzpicture}
        \caption{Input-output plot of $N_6$.}
        \label{fig:OverviewPolytopePatchedDDNNPlot}
      \end{subfigure}
    \caption{Repaired DDNNs.}
    \label{fig:OverviewPatchedDDNN}
\end{figure}
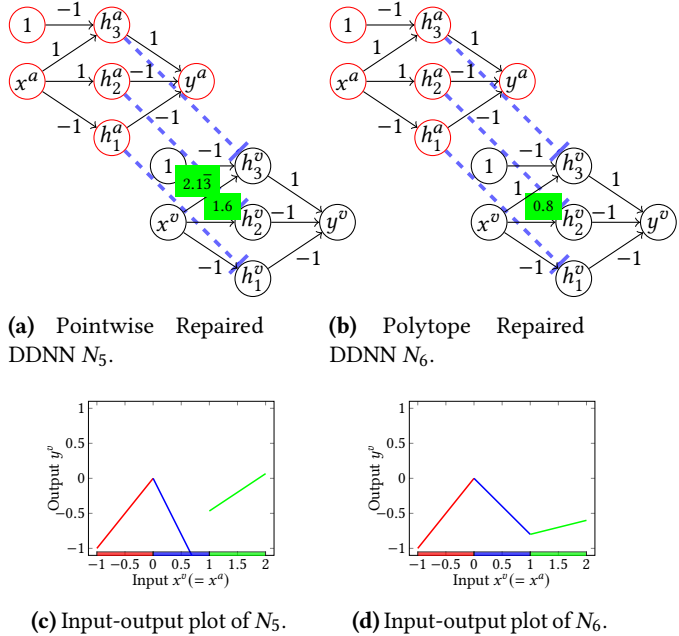

\subsubsubsection{Reduction of polytope repair to pointwise repair.}
Our key insight (\pref{thm:DDNNSameLinRegions}) is that, for piecewise-linear
DDNNs, if we only change the value channel parameters, then we can reduce
polytope repair to pointwise repair. To see this, recall that the value channel
parameters \emph{do not} change the location of the linear regions, only the
behavior within each one. Within each linear region, the behavior of the
network is linear, and, hence, convex. Convexity guarantees that any given polytope is
mapped into another polytope if and only if its \emph{vertices} are mapped into
that polytope.
Note that the assumption of piecewise-linearity is important here: in contrast
to pointwise patching, which works for any feed-forward DNN, our polytope
patching algorithm requires the activation functions to be piecewise-linear.

In our 1D example, this observation is the fact that a line lies in the desired
interval of $[-0.8, -0.4]$ if and only if its endpoints do. In fact, the input
region of interest in our example of $[0.5, 1.5]$ overlaps with two of these
lines (the {\color{blue}blue} and {\color{green}green} line segments in~\pref{fig:OverviewDNNPlot}).
Hence, we must ensure that both of those lines have endpoints in $[-0.8,
-0.4]$.

Thus, the polytope specification is met if and only if the point specification
with $K = \{ K_1 = 0.5, K_2 = 1, K_3 = 1, K_4 = 1.5 \}$ and
$
    A^{K_1} = A^{K_2} = A^{K_3} = A^{K_4} = {\footnotesize \begin{bmatrix} 1 \\ -1 \end{bmatrix}},
$
$
    b^{K_1} = b^{K_2} = b^{K_3} = b^{K_4} = {\footnotesize \begin{bmatrix} -0.4 \\ 0.8 \end{bmatrix}}
$
is met. We call the points in $K$ \emph{key points} because the behavior of the
repaired network $N'$ on these points determines the behavior of the network on
all of $P_1$.

Note that $K_2$ and $K_3$ both refer to the same input point, $1$. This is
because we need to verify that $N'(1)$ is in the desired output range when
approaching either from the left or the right, as we want to verify it for both
the blue and the green lines in~\pref{fig:DecoupledDNNPlot}.  This technicality
is discussed in more detail in~\onlyfor{arxiv}{\pref{app:VertexJacobians}}{the
extended version of this paper~\citep{extendedpaper}}.

Therefore, we have reduced the problem of repair on polytopes to repair on
finitely-many \emph{key points,} which are the vertices of the polytopes in the
specification intersected with the polytopes defining the linear regions of the
DNN. We can apply the algorithm discussed for pointwise repair to solve for a
minimal fix to the first layer. In particular, we get the linear constraints:
$
    -0.8 \leq -0.5 - 0.5\Delta_2 \leq -0.4,
$
$
    -0.8 \leq -1 - \Delta_2 \leq -0.4,
$
$
    -0.8 \leq -1 - \Delta_2 + \Delta_3 + \Delta_4 \leq -0.4,
$
and
$
    -0.8 \leq -1 - 1.5\Delta_2 + 1.5\Delta_3 + \Delta_4 \leq -0.4,
$
for which an $\ell_1$-minimal solution is the single weight change
$
    \Delta_2 = -0.2.
$
The corresponding repaired DDNN is shown
in~\pref{fig:OverviewPolytopePatchedDDNN} and plotted
in~\pref{fig:OverviewPolytopePatchedDDNNPlot}, which shows that the repaired
network satisfies the constraints.

\section{Decoupled DNNs}
\label{sec:DDNNs}

In this section, we formally define the notion of a \emph{Decoupled Deep Neural
Network} (DDNN), which is a novel DNN architecture that will allow for
polynomial-time layer repair.

A DDNN is defined similarly to a DNN (\pref{def:DNN}), except it has two sets
of weights; the \emph{activation channel} has weights $W^{(a, i)}$ and the
\emph{value channel} has weights $W^{(v, i)}$.
\begin{definition}
    \label{def:DDNN}
    A \emph{Decoupled DNN} (DDNN) having layers of size $s_0,
    \ldots, s_n$ is a list of triples $(W^{(a, 1)}, W^{(v, 1)}, \sigma^{(1)})$,
    $\ldots$, $(W^{(a, n)}, W^{(v, n)}, \sigma^{(n)})$, where $W^{(a, i)}$
    and $W^{(v, i)}$ are $s_i \times s_{i-1}$ matrices and $\sigma^{(i)} :
    \mathbb{R}^{s_i}\to \mathbb{R}^{s_i}$ is some \emph{activation function}.
\end{definition}

We now give the semantics for a DDNN. The input $\vec{v}$ is duplicated to form
the inputs $\vec{v}^{(a, 0)}$ and $\vec{v}^{(v, 0)}$ to the activation and
value channels, respectively. The semantics of the activation channel, having
\emph{activation vectors} $\vec{v}^{(a, i)}$, is the same as for a DNN
(\pref{def:DNNFunction}). The semantics for the value channel with
\emph{value vectors} $\vec{v}^{(v, i)}$ is similar, except instead of using the
activation function $\sigma^{(i)}$, we use the \emph{linearization} of
$\sigma^{(i)}$ around the input $W^{(a, i)}\vec{v}^{(a, i-1)}$ of the
corresponding activation layer, as defined
below.
\begin{definition}
    \label{def:Linearize}
    Given function $f : \mathbb{R}^n \to \mathbb{R}^m$ differentiable at
    $\vec{v_0}$, define the \emph{Linearization of $f$ around $\vec{v_0}$} to
    be the function:
    $
        Linearize[f, \vec{v_0}](\vec{x})
        \coloneqq f(\vec{v_0}) + D_{\vec{v}} f(\vec{v_0}) \times(\vec{x} - \vec{v_0}).
    $
\end{definition}
Above, $D_{\vec{v}} f(\vec{v}_0)$ is the Jacobian of $f$ with respect to its
input at the point point $\vec{v}_0$. The Jacobian generalizes the notion of a
scalar derivative to vector functions (see~\onlyfor{arxiv}{\pref{def:Jacobian}}{the extended version of this paper~\citep{extendedpaper}}).  The output of
the DDNN is taken to be the output $\vec{v}^{(v, n)}$ of the value channel.
These DDNN semantics are stated below.
\begin{definition}
    \label{def:DDNNFunction}
    The
    \emph{function $N: \mathbb{R}^{s_0}\to \mathbb{R}^{s_n}$ associated with
    the DDNN} $N$ with layers $(W^{(a,i)}, W^{(v,i)}, \sigma^{(i)})$ is given by
    $
        N(\vec{v}) = \vec{v}^{(v, n)}
    $
    where\\
    $\vec{v}^{(a, 0)} \coloneqq \vec{v}^{(v, 0)} \coloneqq \vec{v}$, \\
    $\vec{v}^{(a, i)} \coloneqq \sigma^{(i)}(W^{(a, i)}\vec{v}^{(a, i-1)})$, and \\
    $\vec{v}^{(v, i)} \coloneqq Linearize[\sigma^{(i)}, W^{(a, i)}\vec{v}^{(a,
    i-1)}](W^{(v, i)}\vec{v}^{(v, i-1)})$.
\end{definition}
DDNNs can be extended to non-differentiable activation functions as discussed
in~\onlyfor{arxiv}{\pref{app:NonDfbl}}{~\citet{extendedpaper}}.

\begin{figure}[t]
    \begin{subfigure}{0.4\linewidth}
        \hspace{-1cm}
        \begin{tikzpicture}
            [inoutdashed/.style={dashed,black,line width=0.1em},
             point/.style={circle,draw,fill=black,minimum size=0.5em},
             relu/.style={blue,line width=0.3em},
             linearized/.style={dashed,red,line width=0.3em},
             scale=0.6]
            \begin{axis}[
                xmin=-2.5, xmax=2.5,
                ymin=-1.5, ymax=1.5,
                axis lines=center,
                axis on top=true,
                domain=-2.5:2.5,font=\LARGE]

                \addplot [mark=none,draw=blue,ultra thick] {max(0, \x)};
                \addplot [mark=none,draw=orange,ultra thick,dashed] {\x};

                \node[point,red] at (axis cs: 1, 0) (ina) {};
                \node[below=0.5em of ina] {$W^{(a, i)}\vec{v}^{(a, i-1)}$};

                \draw[inoutdashed] (axis cs: 1, 0) -- (axis cs: 1, 1);

                \node[point,red] at (axis cs: 1, 1) (outa) {};
                \node[above=0.5em of outa] {$\vec{v}^{(a, i)}$};

                \node[point] at (axis cs: -1, 0) (inv) {};
                \node[above=0.5em of inv] {$W^{(v, i)}\vec{v}^{(v, i-1)}$};

                \draw[inoutdashed] (axis cs: -1, 0) -- (axis cs: -1, -1);

                \node[point] at (axis cs: -1, -1) (inv) {};
                \node[below=0.5em of inv] {$\vec{v}^{(v, i)}$};
            \end{axis}
        \end{tikzpicture}
        \caption{}
        \label{fig:ReLULinearize}
    \end{subfigure}
    \begin{subfigure}{0.4\linewidth}
        \centering
        \begin{tikzpicture}
            [inoutdashed/.style={dashed,black,line width=0.1em},
             point/.style={circle,draw,fill=black,minimum size=0.5em},
             relu/.style={blue,line width=0.3em},
             linearized/.style={dashed,red,line width=0.3em},
             scale=0.6]

            \begin{axis}[
                xmin=-2.5, xmax=2.5,
                ymin=-1.5, ymax=1.5,
                axis lines=center,
                axis on top=true,
                domain=-2.5:2.5,font=\LARGE]

                \addplot [mark=none,draw=blue,ultra thick] {tanh(\x)};
                \addplot [mark=none,draw=orange,ultra thick,dashed] {(\x + 1)/(cosh(-1))^2 + tanh(-1)};

                \node[point,red] at (axis cs: -1, 0) (ina) {};
                \node[above=0.5em of ina] {$W^{(a, i)}\vec{v}^{(a, i-1)}$};

                \draw[inoutdashed] (axis cs: -1, 0) -- (axis cs: -1, -0.761594);

                \node[point,red] at (axis cs: -1, -0.761594) (outa) {};
                \node[below=0.5em of outa] {$\vec{v}^{(a, i)}$};

                \node[point] at (axis cs: 0.2, 0) (inv) {};
                \node[above right=0.2em and 0.15em of inv,
                      fill=white,fill opacity=0.7,text opacity=1] {$W^{(i)}\vec{v}^{(v, i-1)}$};

                \draw[inoutdashed] (axis cs: 0.2, 0) -- (axis cs: 0.2, -0.255);

                \node[point] at (axis cs: 0.2, -0.255) (outv) {};
                \node[below right=0.05em and -0.75em of outv] {$\vec{v}^{(v, i)}$};
            \end{axis}
        \end{tikzpicture}
        \caption{}
        \label{fig:TanhLinearize}
    \end{subfigure}
    \vspace{-.5ex}
    \caption{(a) Linearized ReLU and (b) Linearized Tanh.}
    \vspace{-1ex}
\end{figure}
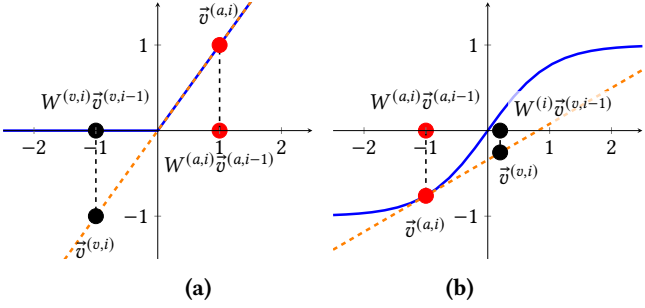

Consider the ReLU activation function in~\pref{fig:ReLULinearize}. We see that
the activation node gets an input of 1 (in red on the $x$ axis) and so produces
an output (in the activation channel) of~$1$. The linearization of the ReLU
function around the point 1 is the identity function $f(x) = x$ (shown in
orange). Thus, we use the function $f(x) = x$ as the activation function for
the corresponding value node. This means that if the value node gets an input
of, say,~$-1$ as shown in black in~\pref{fig:ReLULinearize}, then its output will
be~$-1$. Effectively, if the input to the \emph{activation node} is positive,
then the corresponding value node will be activated (i.e., pass its input
through as its output). On the other hand, if the input to the activation node
were negative, then the linearization would be the zero function $f(x) = 0$.
The value node would use that as its activation function, effectively
deactivating it regardless of its input from the value channel.

Consider also the Tanh activation function (\pref{fig:TanhLinearize}).
The activation channel behaves as normal, each
node outputting the Tanh of its input. For example, if the input to the
activation node is -1 (shown in red), then its output is $\tanh(-1)$ (shown
below it in red). However, for the value channel, we use the linearization of
$\tanh$ around the input to the corresponding activation node. In this case, we
linearize Tanh around -1 to get the line shown in orange, which is used as the
activation function for the value channel.

Thus, as we have shown, each node in the activation channel produces a new
activation function to be used in the value channel.

\subsubsubsection{Key results.}
The first key result shows that the class of DDNNs generalizes that of DNNs;
for any DNN, the below theorem gives a trivial construction for an exactly
equivalent DDNN by setting the activation and value channel weights to be
identical to the weights of the DNN.
\begin{theorem}
    \label{thm:CoupledDDNN}
    Let $N$ be a DNN with layers $(W^{(i)}, \sigma^{(i)})$ and $M$ be the DDNN
    with layers $(W^{(i)}, W^{(i)}, \sigma^{(i)})$. Then, as functions, $N =
    M$.
\end{theorem}

\begin{proof}
    Let $\vec{v}$ be chosen arbitrarily. Let $\vec{v}^{(i)}$ be the
    intermediates of $N$ on $\vec{v}$ according to~\pref{def:DNNFunction}, and
    $\vec{v}^{(a, i)}$, $\vec{v}^{(v, i)}$ be the intermediates of $M$ on
    $\vec{v}$ according to~\pref{def:DDNNFunction}.

    We now prove, for all $i$, that $\vec{v}^{(v, i)} = \vec{v}^{(a, i)} =
    \vec{v}^{(i)}$. We proceed by induction on $i$. By definition,
    $\vec{v}^{(a, 0)} = \vec{v}^{(v, 0)} = \vec{v}^{(0)}$. Now, suppose for
    sake of induction that $\vec{v}^{(a, i)} = \vec{v}^{(v, i)} =
    \vec{v}^{(i)}$. Then we have by definition and the inductive hypothesis
    \[
        \vec{v}^{(a, i+1)}
        = \sigma^{(i+1)}(W^{(i+1)}\vec{v}^{(a, i)})
        = \sigma^{(i+1)}(W^{(i+1)}\vec{v}^{(i)})
        = \vec{v}^{(i+1)},
    \]
    as well as
    \[
        \small
        \begin{aligned}
            &\vec{v}^{(v, i+1)} \\
            =\quad &Linearize[\sigma^{(i+1)}, W^{(i+1)}\vec{v}^{(a, i)}](W^{(i+1)}\vec{v}^{(v, i)}) &&\text{(Definition)} \\
            =\quad &Linearize[\sigma^{(i+1)}, W^{(i+1)}\vec{v}^{(i)}](W^{(i+1)}\vec{v}^{(i)}) &&\text{(Ind. Hyp.)} \\
            =\quad &\sigma^{(i+1)}(W^{(i+1)}\vec{v}^{(i)}) &&\text{(Linearization)} \\
            =\quad &\vec{v}^{(i+1)}, &&\text{(Definition)}
        \end{aligned}
    \]
    because linearizations are exact at their center point.  By induction,
    then, $\vec{v}^{(v, i)} = \vec{v}^{(i)}$ for $0 \leq i \leq n$, and in
    particular $\vec{v}^{(v, n)} = \vec{v}^{(n)}$. But this is by definition
    $M(\vec{v}) = N(\vec{v})$, and as $\vec{v}$ was chosen arbitrarily, this
    gives us $N = M$ as functions.
\end{proof}

Our next result proves that the output of a DDNN varies
linearly with changes in any given value channel layer weights. Note that DDNNs
are \emph{not} linear functions with respect to their \emph{input}, only with
respect to the \emph{value weights.}
\begin{theorem}
    \label{thm:LinearDDNN}
    Let $j$ be a fixed index and $N$ be DDNN with layers $(W^{(a, i)}, W^{(v, i)}, \sigma^{(i)})$. Then, for any $\vec{v}$, $N(\vec{v})$ varies linearly
    \emph{as a function of $W^{(v, j)}$}.
\end{theorem}
\begin{proof}
    Changing $W^{(v, j)}$ does not modify the values of $\vec{v}^{(a, i)}$ or
    $\vec{v}^{(v, i)}$ for $i < j$, hence (i) we can assume WLOG that $j = 1$,
    and (ii) all of the $\vec{v}^{(a, i)}$s remain constant as we vary $W^{(v,
    1)}$.
    Consider now the value of \\
    $
        \vec{v}^{(v, 1)} =
            Linearize[\sigma^{(1)}, W^{(a, 1)}\vec{v}^{(a, 0)}](W^{(v, 1)}\vec{v}^{(v, 0)}).
    $
    This is by definition an linear function of $W^{(v,
    1)}\vec{v}^{(v, 0)}$, which is in turn an linear function of $W^{(v, 1)}$.

    Now, consider any $i > 1$. We have by definition
    $
        \vec{v}^{(v, i)} =
            Linearize[\sigma^{(i)}, W^{(a, i)}\vec{v}^{(a, i-1)}](W^{(v, i)}\vec{v}^{(v, i-1)}),
    $
    which, because we are fixing $W^{(v, i)}$ for $i > 1$, is an linear
    function \emph{with respect to $\vec{v}^{(v, i-1)}$.}

    We showed that $\vec{v}^{(1)}$ is linear \emph{with respect to $W^{(v,
    1)}$}, while $\vec{v}^{(v, i)}$ for $i > 1$ is linear \emph{with respect to
    $v^{(v, i-1)}$.} But compositions of linear functions are also linear,
    hence in total $\vec{v}^{(v, n)}$ is linear with respect to $W^{(v, 1)}$ as
    claimed.
\end{proof}

Our final result proves that modifying only the value
weights in a DDNN does \emph{not} change its linear regions.
\begin{theorem}
    \label{thm:DDNNSameLinRegions}
    Let $N$ be a PWL DNN with layers $(W^{(i)}, \sigma^{(i)})$ and define a
    DDNN $M$ with layers $(W^{(i)}, W^{(v, i)}, \sigma^{(i)})$. Then, within
    any linear region in $LinRegions(N)$, $M$ is also linear.
\end{theorem}
\begin{proof}
    Within any linear region of $N$, all of the activations are the same. This
    means that all of the linearizations used in the computation of
    $\vec{v}^{(v, i+1)}$ do not change in the linear region. Therefore, considering only the value
    channel, we can write $M(\vec{v}) = \vec{v}^{(v, n)}$ as a concatenation of
    linear functions, which is linear with respect to the input.
\end{proof}

\section{Provable Pointwise Repair}
\label{sec:PointPatching}

This section defines and gives an algorithm for provable pointwise repair.

\begin{definition}
    Let $N : \mathbb{R}^n \to \mathbb{R}^m$. Then 
    $(X, A^{\cdot}, b^{\cdot})$ is a \emph{pointwise repair specification} if
    $X$ is a finite subset of $\mathbb{R}^n$ and for each $x \in X$, $A^x$ is a
    $k_x \times m$ matrix while $b^x$ is a $k_x$-dimensional vector.
\end{definition}

\begin{definition}
    Let $N : \mathbb{R}^n \to \mathbb{R}^m$ and $(X,
    A^{\cdot}, b^{\cdot})$ be some pointwise repair specification. Then
    $N$ \emph{satisfies} $(X, A^{\cdot}, b^{\cdot})$, written $\sats{N}{(X,
    A^{\cdot}, b^{\cdot})}$, if $A^xN(x) \leq b^x$ for every $x \in X$.
\end{definition}

\begin{definition}
    Let $N$ be a DDNN and $(X, A^{\cdot}, b^{\cdot})$ be some pointwise repair
    specification. Then another DDNN $N'$ is a \emph{repair} of $N$ if
    $\sats{N'}{(X, A^{\cdot}, b^{\cdot})}$. It is a \emph{minimal repair} of
    $N$ if $\abs{\theta' - \theta}$ is minimal among all repairs, where
    $\theta$ and $\theta'$ are parameters of $N$ and $N'$, respectively, and
    $\abs{\cdot}$ is some user-defined measure of size (e.g., $\ell_1$ norm).
    It is a \emph{minimal layer repair} if it is minimal among all repairs that
    only modify a single, given layer.
\end{definition}

\subsubsubsection{Assumptions on the DNN.}
For point repair, we require only that the activation functions be
(almost-everywhere) differentiable so that we can compute a Jacobian. This is
already the case for every DNN trained via gradient descent, however even this
requirement can be dropped with slight modification to the algorithm
(see~\onlyfor{arxiv}{\pref{app:NonDfbl}}{the extended version of this paper~\citep{extendedpaper}}).  Of particular note, we \emph{do not} require that
the DNN be piecewise-linear.

\begin{algorithm}[t]
    \small
    \DontPrintSemicolon
    \KwIn{
        A DNN $N$ defined by a list of its layers. \\
        A layer index $i$ to repair. \\
        A finite set of points $X$. \\
        For each point $x \in X$ a specification $A^x, b^x$ asserting $A^xN'(x)
        \leq b^x$ where $N'$ is the repaired network.
    }
    \KwOut{
        A repaired DDNN $N'$ or $\bot$.
    }
    \tcc{$C$ is a set of linear constraints on the parameter delta $\vec{\Delta}$
    each of the form $(A, b)$ asserting $A\vec{\Delta} \leq b$.}
    $C \gets \emptyset$\;
    \tcc{Decouple the activation, value layers}
    $N^{'a}, N^{'v} = copy(N), copy(N)$\;
    \tcc{Construct DDNN $N'$ equivalent to DNN $N$}
    $N' \gets DecoupledNetwork(N^{'a}, N^{'v})$\;
    \For{$x \in X$}{
        \tcc{Jacobian wrt parameters of layer $N^{'v}_i$}
        $J^x \gets D_{params(N^{'v}_i)} N'(x)$\;
        \tcc{Encoded constraint $A^x(N(x) + J^x\vec{\Delta}) \leq b^x$}
        $C \gets C \cup \{ (A^xJ^x, b^x - A^xN(x)) \}$\;
    }
    $\vec{\Delta} \gets Solve(C)$\;
    \lIf{$\vec{\Delta} = \bot$}{
        \returnKw{$\bot$}
    }
    \tcc{Update value layer $i$.}
    $params(N^{'v}_i) \gets params(N^{'v}_i) + \vec{\Delta}$\;
    \returnKw{$N'$}\;

    \caption{$\PointR(N, i, X, A^\cdot, b^\cdot)$}
    \label{alg:pointwise}
\end{algorithm}

\subsubsubsection{Algorithm.}
\pref{alg:pointwise} presents our pointwise repair algorithm, which 
reduces provable pointwise repair to an LP.
$params(L)$ returns the parameters of layer $L$.
The notation $D_{params(N^{'v}_i)} N'(x)$ refers to the Jacobian of
the DDNN $N'(x)$ as a function of the parameters $params(N^{'v}_i)$ of the
$i^{\mathrm{th}}$ value channel layer, i.e., $W^{(v, i)}$, while fixing input $x$. Given a set of affine
constraints $C$, $Solve(C)$ returns a solution to the set of constraints or
$\bot$ if the constraints are infeasible.  $Solve$ also guarantees to return
the \emph{optimal} solution according to some user-defined objective function,
e.g., minimizing the $\ell_1$ or $\ell_\infty$ norm. Finally,
$DecoupledNetwork(a, v)$ constructs a decoupled neural network with activation
layers $a$ and value layers $v$.
The next two theorems show the correctness, minimality, and running time of~\pref{alg:pointwise}.

\begin{theorem}
    \label{thm:PointCorrect}
    Given a DNN $N$, a layer index $i$, and a point repair specification $(X,
    A^\cdot, b^\cdot)$, let $N' = \PointR(N, i, X, A^\cdot, b^\cdot)$.
    If $N' \neq \bot$, then $\sats{N'}{(X,
    A^{\cdot}, b^{\cdot})}$ and $\vec{\Delta}$ is a minimal layer repair.
    Otherwise, if $N' = \bot$, then no such single-layer DDNN repair satisfying
    the specification exists for the $i^{\mathit{th}}$ layer.
\end{theorem}

\begin{proof}
    Lines 2--3 construct a DDNN $N'$ equivalent to the DNN
    $N$~(\pref{thm:CoupledDDNN}). For each point $x \in X$, line 5 considers
    the \emph{linearization of $N'$ around the parameters} for the value
    channel layer $N^{'v}_i$, namely:
    $
        N'(x; \vec{\Delta}) \approx N'(x; 0) + J^x\vec{\Delta}
    $
    where $N'(x; \vec{\Delta})$ is the output of the DDNN when the parameters of the
    $i$th layer are changed by $\vec{\Delta}$. In particular,
    by~\pref{thm:CoupledDDNN} if $\vec{\Delta} = 0$, $N'$ is equivalent to $N$ (as a
    function). Hence $N'(x; \vec{\Delta}) \approx N(x) + J^x\vec{\Delta}$.  Finally,
    according to~\pref{thm:LinearDDNN}, this linear approximation is
    \emph{exact} for the DDNN when we only modify the parameters for a single
    value channel layer $N^{'v}_i$, i.e., $N'(x; \vec{\Delta}) = N(x) + J^x\vec{\Delta}$.

    Thus, after the for loop, the set $C$ contains constraints asserting that
    $A^xN'(x; \vec{\Delta}) \leq b^x$, which are exactly the constraints which our
    algorithm needs to guarantee.  Finally, we solve the constraints for
    $\vec{\Delta}$ using an LP solver and return the final DDNN. Hence, if $Solve$
    returns $\bot$, there is no satisfying repair. If it returns a repair, then
    the LP solver guarantees that it satisfies the constraints and no smaller
    $\vec{\Delta}$ exists.
\end{proof}

\begin{theorem}
    \label{thm:PointFast}
    \pref{alg:pointwise} halts in polynomial time with respect to the size of
    the point repair specification $(X, A^\cdot, b^\cdot)$.
\end{theorem}

\begin{proof}
    The LP corresponding to $C$ has one row per row of $A^x$ and one column per
    weight in $N^{'v}_i$, both of which are included in the size of the input.
    Thus, as LPs can be solved in polynomial time, the desired result follows.
\end{proof}
Notably, the above proof assumes the Jacobian computation on line 5 takes
polynomial time; this is the case for all common activation functions used in
practice. The authors are not aware of any actual or proposed activation
function that would violate this assumption.

\section{Provable Polytope Repair}
\label{sec:PolytopePatching}

This section defines and gives an algorithm for provable polytope repair.

\begin{definition}
    Let $N : \mathbb{R}^n \to \mathbb{R}^m$. Then 
    $(X, A^{\cdot}, b^{\cdot})$ is a \emph{polytope repair specification} if
    $X$ is a finite set of bounded convex polytopes in $\mathbb{R}^n$ and for
    each $P \in X$, $A^P$ is a $k_P \times m$ matrix while $b^P$ is a
    $k_P$-dimensional vector.
\end{definition}

\begin{definition}
    Let $N : \mathbb{R}^n \to \mathbb{R}^m$ and $(X,
    A^{\cdot}, b^{\cdot})$ be a polytope repair specification. Then
    $N$ \emph{satisfies} $(X, A^{\cdot}, b^{\cdot})$, written $\satsp{N}{(X,
    A^{\cdot}, b^{\cdot})}$, if $A^PN(x) \leq b^P$ for every $P \in X$ and $x
    \in P$.
\end{definition}

\begin{definition}
    Let $N$ be a DDNN and $(X, A^{\cdot}, b^{\cdot})$ be some polytope repair
    specification. Then another DDNN $N'$ is a \emph{repair} of $N$ if
    $\satsp{N'}{(X, A^{\cdot}, b^{\cdot})}$. It is a \emph{minimal repair} of
    $N$ if $\abs{\theta' - \theta}$ is minimal among all repairs, where
    $\theta$ and $\theta'$ are parameters of $N$ and $N'$ respectively, and
    $\abs{\cdot}$ is some user-defined measure of size (e.g., $\ell_1$ norm).
    It is a \emph{minimal layer repair} if it is minimal among all repairs that
    only modify a single, given layer.
\end{definition}
The quantification in the definition of $\satsp{N}{(X, A^{\cdot},
b^{\cdot})}$ is over an \emph{infinite} set of points $x \in P$.  The rest of
this section is dedicated to \emph{reducing} this infinite quantification to an
equivalent finite one.

\subsubsubsection{Assumptions on the DNN.}
For polytope repair, we assume that the activation functions used by the DNN
are \emph{piecewise-linear} (\pref{def:PWL}). This allows us to \emph{exactly}
reduce polytope repair to point repair, meaning a satisfying repair exists
if and only if the corresponding point repair problem has a solution.
\pref{sec:Conclusion} discusses future work on extending this approach to
non-PWL activation functions.

\subsubsubsection{Algorithm.}
Our polytope repair algorithm is presented in~\pref{alg:polytopewise}. We
reduce the polytope specification $(X, A^{\cdot}, b^{\cdot})$ to a
provably-equivalent \emph{point} specification $(X', A^{'\cdot}, b^{'\cdot})$.
For each polytope in the polytope repair specification, we assert the same
constraints in the point specification except only on the \emph{vertices} of
the linear regions of $N$ on that polytope.
The next two theorems show the correctness, minimality, and running time of~\pref{alg:polytopewise}.
\begin{theorem}
    \label{thm:PolyCorrect}
    Let $N' = \PolyR(N, i, X, A^\cdot, b^\cdot)$ for a
    given DNN $N$, layer index $i$, and polytope repair specification
    $(X, A^\cdot, b^\cdot)$.
    If $N' \neq \bot$, then $\satsp{N'}{(X, A^{\cdot}, b^{\cdot})}$
    and $\vec{\Delta}$ is a minimal layer repair.
    Otherwise, if $N' = \bot$, then no such single-layer DDNN repair satisfying
    the specification exists for the $i^{\mathit{th}}$ layer.
\end{theorem}
\begin{proof}
    Consider an arbitrary $P \in X$ and arbitrary linear region $R \in
    LinRegions(N, P)$. By~\pref{thm:DDNNSameLinRegions}, linear regions are the
    same in the original network $N$ and repaired network $N'$. Hence $R$ is
    also a linear region in $N'$. Thus, on $R$, $N'$ is equivalent to some
    linear function. It is a known result in convex geometry that linear
    functions map polytopes to polytopes and vertices to vertices, i.e., the
    vertices of the postimage $N'(R)$ are given by $N'(v)$ for each vertex $v$
    of $R$. The polytope $N'(R)$ is contained in another polytope if and only
    if its vertices are. Therefore, $N'(R)$ is contained in the polytope
    defined by $A^P, b^P$ if and only if its vertices $N'(v)$ are.  Because
    $P$ and $R$ were chosen arbitrarily and contain all of $X$, the constructed
    point repair specification $(X', A^{'\cdot}, b^{'\cdot})$ is equivalent
    to the polytope repair specification $(X, A^\cdot, b^\cdot)$.  The
    claimed results then follow directly from~\pref{thm:PointCorrect}.
\end{proof}

\begin{theorem}
    \label{thm:PolyFast}
    \pref{alg:polytopewise} halts in polynomial time with respect to the size
    of the polytope repair specification $(X, A^\cdot, b^\cdot)$ and the
    number of vertex points $v$.
\end{theorem}

\begin{proof}
    This follows directly from the time bounds we have established earlier on
    $\PointR$ in~\pref{thm:PointFast}.
\end{proof}

\begin{algorithm}[t]
    \small
    \DontPrintSemicolon
    \KwIn{
        A piecewise-linear DNN $N$ with $n$ inputs and $m$ outputs defined by a
        list of its layers. \\
        A layer index $i$ to repair. \\
        A finite set of polytopes $X$. \\
        For each polytope $P \in X$ a specification $A^P, b^P$ asserting $A^PN'(x)
        \leq b^P$ for every $x \in P$ where $N'$ is the repaired network.
    }
    \KwOut{
        A repaired DDNN $N'$ or $\bot$.
    }
    \tcc{Point repair specification}
    $(X', A^{'\cdot}, b^{'\cdot}) \gets (\emptyset, \emptyset, \emptyset)$\;
    \For{$P \in X$}{
        \For{$R \in LinRegions(N, P)$}{
            \For{$v \in Vertices(R)$}{
                $X'.push(v)$\;
                $A^{'v}, b^{'v} \gets A^P, b^P$\;
            }
        }
    }
    \returnKw{$\PointR(N, i, X', A^{'\cdot}, b^{'\cdot})$}

    \caption{$\PolyR(N, i, X, A^{\cdot}, b^{\cdot})$}
    \label{alg:polytopewise}
\end{algorithm}

The running time of \pref{alg:polytopewise} depends on the number of linear
regions and the number of vertices in each region (line 4). Although in the
worst case there are exponentially-many such linear regions, theoretical results
indicate that, for an $n$-dimensional polytope $P$ and network $N$ with $m$
nodes, we expect $\abs{LinRegions(N, P)} =
O(m^n)$~\citep{regions2019icml,regions2019neurips}.
\citet{syrenn} show efficient computation of one- and two-dimensional
$LinRegions$ for real-world networks.

\section{Experimental Evaluation}
\label{sec:ExperimentalEvaluation}

In this section, we study the efficacy and efficiency of Provable Repair~(PR) on
three different tasks. The experiments were designed to answer the following
questions:
\begin{itemize}
    \item[\RQ{1}] How \emph{effective} is PR in finding
        weights that satisfy the repair specification?
    \item[\RQ{2}] How much does PR cause performance \emph{drawdown}, on
        regions of the input space not repaired?
    \item[\RQ{3}] How well do repairs \emph{generalize} to enforce analogous
        specifications on the input space not directly repaired?
    \item[\RQ{4}] How \emph{efficient} is PR on different
        networks and dimensionalities, and where is most of the time spent?
\end{itemize}

\subsubsubsection{Terms used.}
\emph{Buggy network} is the DNN before repair, while the \emph{fixed} or
\emph{repaired network} is the DNN after repair.
\emph{Repair layer} is the layer of the network that we applied provable repair
to.
\emph{Repair set} is the pointwise repair specification or the
polytope repair specification used to synthesize the repaired network.
\emph{Generalization set} is a set of points (or polytopes) that are disjoint
from but simultaneously similar to the repair specification.
\emph{Drawdown set} is a set of points (or polytopes) that are disjoint from
and not similar to the repair specification.
\emph{Efficacy} is the percent of the repair set that is classified correctly
by the repaired network, i.e., the accuracy of the repaired network on the repair set. Our theoretical guarantees ensure that Provable Repair efficacy is always
100\%.
\emph{Generalization Efficacy} is computed by subtracting the accuracy on the
generalization set of the buggy network from that of the repaired network.
Higher generalization efficacy implies better generalization of the fix.
\emph{Drawdown} is computed by subtracting the accuracy on the drawdown set of
the repaired network from that of the buggy network.  Lower drawdown is
better, implying less forgetting.

\subsubsubsection{Fine-Tuning Baselines.}
We compare Provable Repair (PR) to two baselines.
The first baseline performs fine-tuning (FT) using
gradient descent on all parameters at once, as proposed
by~\cite{sinitsin2020editable}. FT runs gradient descent until all
repair set points are correctly classified.

The second baseline, modified fine-tuning (MFT), is the same as FT except
(a)~MFT fine-tunes only a single layer, (b)~MFT adds a loss term penalizing the
$\ell_0$ and $\ell_{\infty}$ norms of the repair, (c)~MFT reserves $25\%$ of the
repair set as a holdout set, and (d)~it stops once the accuracy on the holdout
set begins to drop. Note that this approach does not achieve full efficacy;
hence, it is not a valid repair algorithm (it does not repair the DNN). However,
because of the early-stopping, MFT should have lower drawdown.

In all cases PR, FT, and MFT were given the same
repair set (which included a number of non-buggy points). However, for polytope
repair it is necessary to sample from that infinite repair set to form a finite repair
set for FT and MFT, using the same number of randomly-sampled points as key
points in the PR algorithm.

\subsubsubsection{Evaluation Platform.}
All experiments were run on an Intel\textsuperscript{\textregistered}
Xeon\textsuperscript{\textregistered} Silver 4216 CPU @~2.10GHz.
BenchExec~\citep{benchexec} was used to ensure reproducibility and limit the
experiment to 32~cores and 300~GB of memory. The PyTorch framework was used for
performing linear algebra
computations~\citep{DBLP:conf/nips/PaszkeGMLBCKLGA19}.  The experiments were
run entirely on CPU. We believe that performance can be improved (i)~by
utilizing GPUs and (ii)~by using
TensorFlow~\citep{DBLP:conf/osdi/AbadiBCCDDDGIIK16}, which has explicit support
for Jacobian computations. We used Gurobi~\citep{gurobi} to solve the LP
problems. The code to reproduce our experimental results is available at
\url{https://github.com/95616ARG/PRDNN}.

\subsection{\task{1}: Pointwise ImageNet Repair}
\label{sec:EvaluationImageNet}

\subsubsubsection{Buggy network.}
SqueezeNet~\citep{SqueezeNet}, a modern ImageNet convolutional neural network.
We slightly modified the standard model~\citep{ONNXNetworks}, removing all
output nodes except for those of the 9 classes used (see below). The resulting
network has 18 layers, 727,626 parameters, and an accuracy of
93.6\% on these classes using the official ImageNet validation set.

\subsubsubsection{Repair set.} The Natural Adversarial Examples~(NAE)
dataset, which are images commonly misclassified by
modern ImageNet networks~\citep{hendrycks2019nae}. This dataset was also used
by~\citet{sinitsin2020editable}. For our nine classes (chosen alphabetically
from the 200 total in the NAE dataset), the NAE dataset contains 752 color
images. The buggy network has an accuracy of 18.6\% on these NAE images.  To
measure scalability of repair, we ran 4 separate experiments, using subsets of
100, 200, 400, and all 752 NAE images as the repair specification.

\subsubsubsection{Repair layer.}
PR and MFT was used to repair each of the 10 
feed-forward fully-connected or convolution layers. 
\pref{tab:ImageNet} only lists the PR and MFT results
for the layer with the best drawdown~(BD).

\subsubsubsection{Generalization set.}
The NAE images do not have a common feature that we would like the network to
generalize from the repair. Thus, we were not able to construct a
generalization set to evaluate generalization for \task{1}.

\subsubsubsection{Drawdown set.} The entire set of approximately 500 validation
images for the nine selected classes from the official ImageNet validation
set~\citep{imagenet_cvpr09}.

\subsubsubsection{Fine-tuning hyperparameters.}
Both FT and MFT use standard SGD with a learning rate of $0.0001$ and no
momentum, which were chosen as the best parameters after a small manual search.
FT[1] and MFT[1] use batch size 2, while FT[2] and MFT[2] use batch size 16.

\begin{table*}[t]
    \small
    \caption{Summary of experimental results for \task{1}. D: Drawdown (\%), T:
    Time, BD: Best Drawdown, PR: Provable Repair, FT: Fine-Tuning baseline,
    MFT: Modified Fine-Tuning baseline (best layer), E: Efficacy (\%). Efficacy
    of PR and FT is always 100\%, hence Efficacy (E) numbers are only provided
    for MFT.
    }
    \begin{tabular}{@{}l@{\hskip .3cm}cr@{\hskip .5cm}cr@{\hskip .5cm}cr@{\hskip .5cm}ccr@{\hskip .5cm}ccr}
        & \multicolumn{2}{c}{PR (BD)}
        & \multicolumn{2}{c}{FT[1]}
        & \multicolumn{2}{c}{FT[2]}
        & \multicolumn{3}{c}{MFT[1] (BD)}
        & \multicolumn{3}{c}{MFT[2] (BD)}
        \\
        \cmidrule(r){2-3} \cmidrule(r){4-5} \cmidrule(r){6-7} \cmidrule(r){8-10} \cmidrule(r){11-13}
        Points & D & T & D & T & D & T & E & D & T & E & D & T \\
        \midrule
        100
        & 3.6 & 1m39.0s
        & 10.2 & 4m31.8s
        & 8.2 & 9m24.0s
        & 24 & -0.7 & 19.2s
        & 28 & 0.0 & 4m4.9s
        \\
        200
        & 1.1 & 2m50.8s
        & 9.6 & 12m19.5s
        & 9.6 & 26m35.0s
        & 21.5 & -0.4 & 13.4s
        & 20 & -0.4 & 2m54.5s
        \\
        400
        & 5.1 & 4m45.3s
        & 13.8 & 34m2.6s
        & 11.1 & 1h9m26.8s
        & 21.25 & -0.4 & 29.3s
        & 21 & -0.4 & 1m32.1s
        \\
        752
        & 5.3 & 8m28.1s
        & 15.4 & 1h22m18.7s
        & 13.4 & 2h33m8.2s
        & 19.4 & -0.4 & 2m35.8s
        & 18.2 & -0.4 & 1m46.7s
    \end{tabular}
    \label{tab:ImageNet}
    \vspace{2mm}
\end{table*}

\subsubsubsection{\RQ{1}: Efficacy.}
When using 100, 200, and 400 points, our Provable Repair algorithm was able to
find a satisfying repair for any layer, i.e., achieving 100\% efficacy. For the
752 point experiment, Provable Repair was able to find a satisfying repair
when run on 7 out of the 10 layers. It timed out on one of the layers and on
the other two was able to prove that no such repair exists. FT
was also able to find a 100\%-efficacy repair in all four cases.

Meanwhile, the MFT baseline had efficacy of at most $28\%$, meaning it only
marginally improved the network accuracy on NAE points from the original
accuracy of $18\%$.

\subsubsubsection{\RQ{2}: Drawdown.}
\pref{tab:ImageNet} summarizes the drawdown of Provably Repaired networks on
this task.  In all cases, PR was able to find a layer whose repair resulted in
under $6\%$ drawdown. Extended results are in the
\onlyfor{arxiv}{appendix~\pref{tab:ImageNet2}}{extended version of this paper}. Per-layer drawdown is shown
in~\pref{fig:ImageNetDrawdown}, where we see that for this task repairing
earlier layers can lead to much higher drawdown while latter layers result in
consistently lower drawdown. This suggests a heuristic for repairing ImageNet
networks, namely focusing on latter layers in the network.

By contrast, FT had consistently worse drawdown
across multiple hyperparameter configurations, always above $8\%$ and in some
cases above $15\%$. This highlights how the guarantee of finding a minimal fix
using Provable Repair can lead to significantly more localized fixes,
preventing the DNN from forgetting what it had learned previously as is often a
major risk when fine-tuning.

The MFT baseline had very low drawdown, but this comes at the cost of low
efficacy.

\begin{figure}[b]
    \hspace{-5mm}
    \begin{subfigure}[t]{0.45\linewidth}
        \centering
        \begin{tikzpicture}
            \begin{axis}
            [scale=0.4, xlabel=Repaired Layer,ylabel=Drawdown,ymin=-2,ymax=70,
             y label style={at={(axis description cs:0.2,.5)}},
             x tick label style={font=\tiny},
             symbolic x coords={1,2,4,6,8,10,12,14,16,18},
             xtick={1,2,4,6,8,10,12,14,16,18}]
                \addplot table [x=layer, y=drawdown, col sep=comma]
                {data/imagenet_patching/400_points.csv};
            \end{axis}
        \end{tikzpicture}
        \caption{}
        \label{fig:ImageNetDrawdown}
    \end{subfigure}
    \hspace{1mm}
    \begin{subfigure}[t]{0.45\linewidth}
        \centering
        \begin{tikzpicture}
            \begin{axis}
            [ybar stacked, scale=0.4, xlabel=Repaired Layer, ylabel=Time (s),
                bar width=5, x tick label style={font=\tiny},
             symbolic x coords={1,2,4,6,8,10,12,14,16,18},
             xtick={1,2,4,6,8,10,12,14,16,18}]
                \addplot table [x=layer, y=jacobian, col sep=comma]
                {data/imagenet_patching/400_points.csv};
                \addplot table [x=layer, y=solver, col sep=comma]
                {data/imagenet_patching/400_points.csv};
                \addplot table [x=layer,
                y expr=\thisrow{total} - \thisrow{solver} - \thisrow{jacobian},
                col sep=comma]
                {data/imagenet_patching/400_points.csv};
            \end{axis}
        \end{tikzpicture}
        \caption{}
        \label{fig:ImageNetTime}
    \end{subfigure}
    \caption{(a)~Drawdown and (b)~timing per repair layer when using 400 images
    in the repair set for \task{1}. Blue: Jacobian, Red: Gurobi, Brown: Other.}
\end{figure}
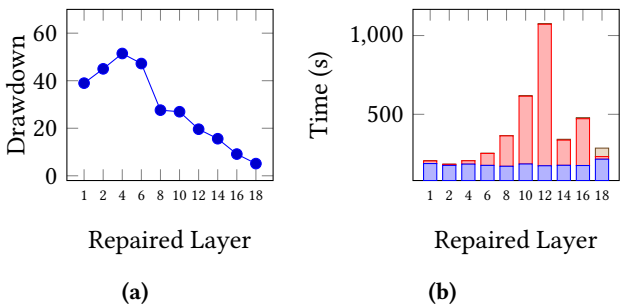

\subsubsubsection{\RQ{4}: Efficiency.}
\pref{tab:ImageNet} also shows the amount of time taken to repair different
numbers of points. Even with all 752 points, PR was able to find the repair
with the best drawdown in under 10 minutes. If the layers are repaired in
parallel, then all single-layer repairs can be completed in 4m, 8m, 18m, and
1h26m for 100, 200, 400, and 752 points respectively. If the layers are
repaired in sequence, all single-layer repairs can be completed instead in 15m,
31m, 1h7m, and 4h10m respectively.
To understand where time was spent,~\pref{fig:ImageNetTime} plots
the time taken (vertical axis) against the layer fixed (horizontal axis) for
the 400-point experiments. We have further divided the time spent into (i) time
computing parameter Jacobians, (ii) time taken for the Gurobi
optimization procedure, and (iii) other. For this model we find
that a significant amount of time is spent computing Jacobians. This is because
PyTorch lacks an optimized method of computing such Jacobians, and so we
resorted to a sub-optimal serialized approach.

By comparison, FT can take significantly longer. For
example, on the 752-point experiment we saw FT take over 2 hours. In
practice, this is highly dependent on the hyperparameters chosen, and
hyperparameter optimization is a major bottleneck for FT
in practice.
The MFT baseline was also fast, but this comes at the cost of low efficacy.

\subsection{\task{2}: 1D Polytope MNIST Repair}
\label{sec:EvaluationMNIST}

\subsubsubsection{Buggy network.} The MNIST ReLU-3-100 DNN from~\citet{ERAN},
consisting of $3$ layers and $88,010$ parameters for classifying handwritten
digits. This network has an accuracy of 96.5\% on the MNIST test set.

\subsubsubsection{Repair set.}
We ran separate experiments with 10, 25, 50, or 100 lines; each line was
constructed by taking as one endpoint an uncorrupted MNIST handwritten digit
image and the other endpoint that same MNIST image corrupted with fog from
MNIST-C~\citep{DBLP:journals/corr/abs-1906-02337}. If $I$ is the uncorrupted
image and $I'$ is the fog-corrupted image, then this specification states that
all (infinite) points along the line from $I$ to $I'$ must have the same
classification as $I$. The buggy network has an accuracy of 20.0\% on the
corrupted endpoints used in the line repair specification.

Note that, unlike Provable Polytope Repair, both FT and MFT are given
finitely-many points sampled from these lines --- they can\emph{not} make any
guarantees about other points on these lines that are not in its sampled set.

\subsubsubsection{Repair layer.}
We ran two repair experiments, repairing each of the last two layers.  Because
the network is fully-connected, the first layer has a very large number of
nodes (because it is reading directly from the 784-dimensional input image),
leading to a very large number of variables in the constraints. By comparison,
SqueezeNet used convolutional layers, so the size of the input does not matter.

\subsubsubsection{Generalization set.}
The MNIST-C fog test set consisting of $10,000$ fog-corrupted MNIST images as
the generalization set.  The accuracy of the buggy network is 19.5\% on this
generalization set.

\subsubsubsection{Drawdown set.}
The drawdown set is the official MNIST test set, which contains $10,000$
(uncorrupted) MNIST images. The images in this test set are the
exactly the uncorrupted versions of those in the generalization set. The buggy
network has 96.5\% accuracy on the drawdown set.

\subsubsubsection{Fine-tuning hyperparameters.}
Both use standard SGD with a batch size of $16$ and momentum $0.9$, chosen as
the best parameters after a small manual search. FT[1] and MFT[1] use a learning rate of
$0.05$ while FT[2] and MFT[2] use a learning rate of $0.01$.

\begin{table*}
    \small
    \caption{Summary of experimental results for \task{2}. D: Drawdown (\%), G:
    Generalization (\%), T: Time, PR: provable repair, FT: fine-tuning
    baseline. $\ast$ means fine-tuning diverged and timed out after 1000
    epochs, the results shown are from the last iteration of fine-tuning before
    the timeout.}
    \begin{tabular}{@{}lc@{\hskip .5cm}ccr@{\hskip .5cm}ccr@{\hskip .5cm}ccr@{\hskip .5cm}ccr}
        &
        & \multicolumn{3}{c}{PR (Layer 2)}
        & \multicolumn{3}{c}{PR (Layer 3)}
        & \multicolumn{3}{c}{FT[1]}
        & \multicolumn{3}{c}{FT[2]}
        \\
        \cmidrule(r){3-5} \cmidrule(r){6-8} \cmidrule(r){9-11} \cmidrule(r){12-14}
        Lines & Points
        & D & G & T
        & D & G & T
        & D & G & T
        & D & G & T
        \\
        \midrule
        10 & 1730
        & 1.3 & 30.7 & 1m55.1s
        & 5.7 & 32.1 & 1.7s
        & 56.0 & 4.2 & 0.4s
        & 8.3 & 27.5 & 0.6s
        \\
        25 & 4314
        & 1.8 & 35.5 & 2m46.5s
        & 5.5 & 38.3 & 3.7s
        & 36.5 & 22.4 & 1.2s
        & 3.8 & 51.0 & 0.4s
        \\
        50 & 8354
        & 2.6 & 38.3 & 4m29.3s
        & 5.9 & 44.5 & 8.0s
        & 85.2$^\ast$ & -8.2$^\ast$ & 29m36.5s$^\ast$
        & 4.7 & 55.8 & 0.8s
        \\
        100 & 16024
        & 2.4 & 42.9 & 10m55.7s
        & 5.9 & 46.0 & 18.4s
        & 31.4 & 37.7 & 3.1s
        & 3.2 & 60.0 & 1.6s
    \end{tabular}
    \label{tab:MNIST}
    \vspace{7mm}
\end{table*}

\begin{table*}
    \small
    \caption{Summary of modified fine-tuning results for \task{2}.
    E: Efficacy (\%), D: Drawdown (\%), G: Generalization (\%), T: Time, MFT:
    modified fine-tuning baseline. Note that the modified fine-tuning
    \emph{does not} satisfy all of the hard constraints; therefore, it is not
    in fact repairing the network. However, it does result in lower drawdown.}
    \begin{tabular}{@{}lc@{\hskip .5cm}cccr@{\hskip .5cm}cccr@{\hskip .5cm}cccr@{\hskip .5cm}cccr}
        & \multicolumn{4}{c}{MFT[1] (Layer 2)}
        & \multicolumn{4}{c}{MFT[1] (Layer 3)}
        & \multicolumn{4}{c}{MFT[2] (Layer 2)}
        & \multicolumn{4}{c}{MFT[2] (Layer 3)}
        \\
        \cmidrule(r){2-5} \cmidrule(r){6-9} \cmidrule(r){10-13} \cmidrule(r){14-17}
        Lines
        & E & D & G & T
        & E & D & G & T
        & E & D & G & T
        & E & D & G & T
        \\
        \midrule
        10
        & 66.5 & 1.9 & 14.3 & 0.7s
        & 60.7 & 0.1 & 3.6 & 0.5s
        & 70.3 & 0.5 & 16.8 & 0.4s
        & 58.4 & -0.05 & 1.3 & 0.5s
        \\
        25
        & 67.3 & 0.6 & 16.4 & 1.0s
        & 57.4 & 0.3 & 2.4 & 38.3s
        & 65.8 & 0.6 & 16.9 & 1.0s
        & 56.1 & 0.03 & 1.0 & 1.0s
        \\
        50
        & 71.3 & 0.6 & 17.9 & 1.7s
        & 61.5 & 0.1 & 1.7 & 1.6s
        & 70.5 & 0.7 & 17.5 & 1.1s
        & 59.7 & 0.1 & 0.8 & 1.6s
        \\
        100
        & 69.7 & 0.6 & 11.9 & 2.2s
        & 63.7 & 0.1 & 2.3 & 2.2s
        & 69.8 & 0.4 & 12.9 & 3.3s
        & 62.7 & 0.05 & 0.5 & 5.2s
    \end{tabular}
    \label{tab:MNIST_MFT}
    \vspace{7mm}
\end{table*}

\pref{tab:MNIST}~and~\pref{tab:MNIST_MFT} summarize the results for \task{2} when repairing Layer~2 and
Layer~3. The ``Lines'' column lists the number of lines in the repair
specification. The ``Points'' column lists the number of key points in the
$LinRegions$, i.e., the size of the constructed $X'$
in~\pref{alg:polytopewise}.

\subsubsubsection{\RQ{1}: Efficacy.}
PR always found a repaired network, i.e.,
one guaranteed to correctly classify all of the infinitely-many points on each of the
lines used in the repair specification.

FT
could usually find a repaired network that achieved 100\% accuracy on its
sampled repair points, but could not make any guarantee about the
infinitely-many other points in the line specification. Furthermore, in one
configuration FT timed out after getting stuck in a very
bad local minima, resulting in a network with near-chance accuracy. This
highlights how extremely sensitive FT is to the choice of
hyperparameters, a major inconvenience when attempting to apply the technique
in practice when compared to our hyperparameter-free LP formulation.

Meanwhile, MFT was only able to achieve at most 71.3\% efficacy, and like FT
this does not ensure anything about the infinitely-many other points in the
specification.

\subsubsubsection{\RQ{2}: Drawdown.}
Provable Repair results in low drawdown on this task.  Repairing
Layer~2 consistently results in less drawdown, with a drawdown of 2.4\% when
repairing using all 100~lines.  However, even when repairing Layer 3 the
drawdown is quite low, always under 6\%. FT has
significantly worse drawdown, up to 56.0\% even when FT
terminates successfully. This highlights how the Provable Repair guarantee of
finding the \emph{minimal} repair significantly minimizes forgetting. Here
again, FT is extremely sensitive to hyperparamaters, with a different
hyperparameter choice often leading to order of magnitude improvements in
drawdown.

MFT achieved low drawdown, but at the cost of worse generalization and
efficacy.

\subsubsubsection{\RQ{3}: Generalization.} Provable Repair results in
significant generalization, improving classification accuracy of fog-corrupted
images not part of the repair set,  regardless of the layer repaired.  For
instance, repairing Layer 3 using 100 lines resulted in generalization of 46\%;
that is, the accuracy improved from 19.5\% for the buggy network to 65.5\% for
the fixed network. In some scenarios FT has slightly
better generalization, however this comes at the cost of higher drawdown.
Furthermore, Provable Repair tends to have better generalization when using
fewer lines (i.e., smaller repair set), which highlights how FT has a
tendency to overfit small training sets. We also note that FT
again shows extreme variability (2--10$\times$) in generalization
performance between different hyperparameter choices, \emph{even when} it
successfully terminates with a repaired network.
MFT had consistently worse generalization than PR, sometimes by multiple orders
of magnitude.

\subsubsubsection{\RQ{4}: Efficiency.}
In addition to the time results in~\pref{tab:MNIST}, we did a deeper analysis
of the time taken by various parts of the repair process for the 100-line
experiment. For Layer 2, repairing completed in $655.7$ seconds, with $1.0$ seconds
taken to compute $LinRegions$, $8.0$ seconds for computing Jacobians, $623.6$
seconds in the LP solver, and $23.1$ seconds in other tasks. For Layer 3,
repairing completed in $18.4$ seconds, with $1.0$ seconds computing $LinRegions$,
$1.0$ seconds computing Jacobians, $12.9$ seconds in the LP solver, and $3.5$
seconds in other tasks.
We find that repairing the lines was quite efficient, with the majority of the
time taken by the Gurobi LP solver. In contrast, for \task{1} the majority of
time was spent computing Jacobians. This is because we implemented an optimized
Jacobian computation for feed-forward networks.

We see that the time taken to repair depends on the particular layer that is
being repaired. There is little overhead from our constraint encoding process
or computing LinRegions. Because the majority of the time is spent in the
Gurobi LP solver, our algorithm will benefit greatly from the active and ongoing
research and engineering efforts in producing significantly faster LP solvers.

FT was also fast. However, again we note that for
some choices of hyperparameters FT gets stuck and cannot
find a repaired network with better-than-chance accuracy. This highlights how
sensitive such gradient-descent-based approaches are to their hyperparameters,
in contrast to our LP-based formulation that is guaranteed to find the minimal
repair, or prove that none exists, in polynomial time.
Finally, MFT was consistently fast but at the cost of consistently worse
efficacy and generalization.

\subsection{\task{3}: 2D Polytope ACAS Xu Repair}
\label{sec:EvaluationACAS}

\subsubsubsection{Buggy network:}
\task{3} uses the $N_{2,9}$ ACAS Xu network~\citep{julian2018deep}, which has 7 layers and
13,350 parameters.
The network takes a five-dimensional
representation of the scenario around the aircraft, and outputs
one of five possible advisories.

\subsubsubsection{Repair set.}
\citet{reluplex:CAV2017} show that $N_{2,9}$ violates
the safety property $\phi_8$. However,
we cannot directly use $\phi_8$ as a polytope repair specification
because (i)~$\phi_8$ concerns a five-dimensional polytope, and
existing techniques for computing $LinRegions$ only scale to two dimensions on
ACAS-sized neural networks, and (ii)~$\phi_8$
specifies that the output advisory can be one of two possibilities,
a disjunction that cannot be encoded as an LP. To circumvent reason~(i), \task{3}
uses 10 randomly-selected two-dimensional planes (slices) that contain
violations to property $\phi_8$. To circumvent reason~(ii), \task{3} strengths
$\phi_8$ based on the existing behavior of the network. For each key
point in the two-dimensional slice, we compute which of the two possibilities was higher in
the buggy network $N_{2,9}$, and use the higher one as the desired output
advisory. Notably, any network that satisfies this strengthened property also
satisfies property $\phi_8$.

\subsubsubsection{Repair layer.}
We used the last layer as the repair layer. The other layers were
unsatisfiable, i.e.,~\pref{alg:polytopewise} returned $\bot$.

\subsubsubsection{Generalization set.}
$5,466$ counterexamples to the safety property $\phi_8$ that were not in the
repair set. These counterexamples were found by computing $LinRegions$ on 12
two-dimensional slices randomly selected from $R$.

\subsubsubsection{Drawdown set.} A similarly randomly-sampled set of $5,466$
points that were correctly classified by the buggy network.
Generalization and drawdown sets have the same size.

\subsubsubsection{Fine-tuning hyperparameters.}
Both FT and MFT use standard SGD with learning rate of $0.001$, momentum $0.9$, and batch size
$16$ chosen as best from a small manual search.

\subsubsubsection{\RQ{1}: Efficacy.}
Provable Polytope Repair was able to provably repair all 10 two-dimensional
slices in the repair set, i.e., synthesize a repaired network that
satisfies safety property $\phi_8$ on all infinitely-many points on the 10
2D repair slices.

By contrast, both FT and MFT had \emph{negative} efficacy; 
viz., while the original network misclassified only 3~points 
in the sampled repair set, the FT-repaired network misclassified 
181 points and the MFT-repaired networks misclassified
between 10 and 50 points. 

\subsubsubsection{\RQ{2}: Drawdown.}
The drawdown for Provable Repair was zero: the fixed network correctly
classified all $5,466$ points in the drawdown set.
By contrast, FT led to $650$ of the $5,466$ points that
were originally classified correctly to now be classified incorrectly. This
highlights again how fine-tuning can often cause forgetting.
For all layers, MFT had a drawdown of less than~$1\%$.

\subsubsubsection{\RQ{3}: Generalization.}
$5,176$ out of $5,466$ points in the generalization set were correctly
classified in the Provably Repaired network; only $290$ were incorrectly
classified. Recall that all $5,466$ were incorrectly classified in the buggy
network. Thus, the generalization is $94.69\%$.

FT left only $216$ points incorrectly classified,
resulting in a slightly better generalization of $95.8\%$. However, this comes
at the cost of introducing new bugs into the network behavior (see Drawdown
above) and failing to achieve 100\% efficacy even on the finitely-many repair
points it was given.
MFT had a generalization of~$100\%$ for the last two layers, and 
a generalization of less than
$10\%$ for the remaining layers.

\subsubsubsection{\RQ{4}: Efficiency.}
It took a total of $21.2$~secs.~to Provably Repair the network using the 10
two-dimensional slices; computing $LinRegions$ took 1.5~secs.; 1.8~secs.~to
compute Jacobians; 7.0~secs.~for the Gurobi LP solver; and 10.9~secs.~for other
tasks.

FT never fully completed; we timed it out after 1000
epochs taking 1h18m9.9s. This highlights the importance of our theoretical
guarantees that Provable Repair will either find the minimal fix or prove that
no fix exists in polynomial time.
MFT completed within 3 seconds for all layers.

\section{Related Work}
\label{sec:RelatedWork}

Closest to this paper is
\citet{DBLP:conf/lpar/GoldbergerKAK20}, which can be viewed as finding minimal
layer-wise fixes for a DNN given a pointwise specification. However, their
algorithm is exponential time and their underlying formulation
is NP-Complete.  By contrast, DDNNs allow us to reduce this repair problem to an LP.  
Furthermore, \citep{DBLP:conf/lpar/GoldbergerKAK20} only
addresses pointwise repair, and not provable
\emph{polytope} repair. 
These issues are demonstrated
in the experimental results; whereas~\citep{DBLP:conf/lpar/GoldbergerKAK20} is
able to repair only 3 points even after running for a few days, we repair
entire polytopes (infinitely many points, reduced to over 150,000 key points)
in under thirty seconds for the same ACAS Xu DNN.
Furthermore,
reliance on Marabou~\citep{DBLP:conf/cav/KatzHIJLLSTWZDK19} means
\citep{DBLP:conf/lpar/GoldbergerKAK20} is restricted to PWL activation
functions. Our provable pointwise repair algorithm is applicable to DNNs using
non-PWL activations such as Tanh and Sigmoid.

\citet{DBLP:conf/ijcnn/KauschkeLF19} focuses on image-recognition models under
distributional shift, but polytope repair is not considered. The technique
learns a predictor that estimates whether the original network will misclassify
an instance, and a repaired network that fixes the misclassification. 

\citet{sinitsin2020editable} proposed \emph{editable neural networks}, which
train the DNN to be easier to manipulate post-training. Unlike our approach,
their technique does not provide provable guarantees of efficacy or minimality,
and does not support polytope repair. When the original training dataset is
unavailable, their approach reduces to the fine-tuning technique we used as a
baseline.
Similarly,~\citet{DBLP:conf/iclr/TramerKPGBM18} injects adversarial examples
into training data to increase robustness.

Robust training techniques~\cite{fischer2019dl2} apply gradient descent to an
abstract interpretation that computes an over-approximation of the model's
output set on some polytope input region. Such approaches have the same
sensitivity to hyperparameters as retraining and fine-tuning techniques.
Furthermore, they use coarse approximations designed for
pointwise robustness. These coarse approximations blow up on the larger input
regions considered in our experiments, making the approach ineffective for
repairing such properties.

GaLU networks~\citep{galu} can be thought of as a variant of decoupled networks
where activations and values get re-coupled after every layer. Thus, multi-layer
GaLU networks do not satisfy the key theoretical properties of DDNNs.

\citet{DBLP:conf/aaai/AlshiekhBEKNT18,vrl} ensure safety of reinforcement
learning controllers by synthesizing a \emph{shield} guaranteeing it satisfies a
temporal logic property. \citet{DBLP:conf/nips/BastaniPS18} present policy
extraction for synthesizing provably robust decision tree policy for deep
reinforcement learning.

Prior work on DNN verification focused on safety
and
robustness~\citep{DBLP:conf/cav/KatzHIJLLSTWZDK19,reluplex:CAV2017,Huang:CAV2017,Ehlers:ATVA2017,Bunel:NIPS2018,Bastani:NIPS2016,ai2:SP2018,Singh:POPL2019,Anderson:PLDI2019}.
More recent research tackles testing of DNNs
\citep{pei2017deepxplore,tian2018deeptest,sun2018concolic,odena2019tensorfuzz,DBLP:conf/issre/MaZSXLJXLLZW18,DBLP:conf/issta/XieMJXCLZLYS19,DBLP:conf/issre/GopinathZWKPK19,DBLP:conf/kbse/MaJZSXLCSLLZW18}.
Our algorithms can fix errors found by such tools.

\section{Conclusion and Future Work}
\label{sec:Conclusion}

We introduced \emph{provable repair} of DNNs, and presented algorithms
for \emph{pointwise} and \emph{polytope} repair; the former handles specifications
on finitely-many inputs, and the latter handles a symbolic specification about
infinite sets of points. We introduced \emph{Decoupled} DNNs,
which allowed us to reduce provable pointwise repair to an LP problem.
For the common class of piecewise-linear DNNs, our polytope
repair algorithm can \emph{provably} reduce the polytope repair problem to a
pointwise repair problem.
Our extensive experimental evaluation on three different tasks
demonstrate that pointwise and
polytope repair are effective, generalize well, display minimal drawdown, and
scale well.

The introduction of provable repairs opens many exciting directions for future
work.
We can employ \emph{sound approximations} of linearizations to improve performance, and
support non-piecewise-linear activation functions for polytope repair.
Repairing multiple layers could be achieved
by using the natural generalization of our LP formulation to a QCQP
\cite{baron1972quadratic}, or by iteratively applying our LP formulation to
different layers.
Future work may repair the
activation parameters, or convert the resulting DDNN back into a standard,
feed-forward DNN while still satisfying the specification (e.g., to reduce the
small computational overhead of the DDNN).
Exploring learning-theoretic properties of the repair process, the trade-off between generalization and
drawdown during repair,
and heuristics for choosing repair layers, are all very interesting and important lines
of future research to make provable repair even more useful.
Future work could explore repairing Recurrent Neural Networks (RNNs)
using linear temporal logic specifications.
Future work may explore how to speed up repair using hardware accelerators
beyond the native support provided by PyTorch.
Finally, experimenting with different objectives, relaxations, or solving
methods may lead to even more efficient mechanisms for DNN repair.

\begin{acks}
  We thank our shepherd Osbert Bastani and the other reviewers for their
  feedback and suggestions. This work is supported in part by NSF grant
  CCF-2048123 and a Facebook Probability and Programming research award.
\end{acks}

\balance
\bibliography{main}

\onlyfor{arxiv}{
\clearpage
\appendix
\section{Linear Approximations of Vector Functions}
\label{app:Calculus}
This section discusses linear approximations of vector functions. The key
definition is \pref{def:Jacobian}, which defines the Jacobian. An example is
given in~\pref{sec:JacobianExample}. Readers comfortable
with~\pref{def:Jacobian} and~\pref{sec:JacobianExample}, i.e., taking the
Jacobian of a vector-valued function with respect to a matrix parameter, are
welcome to skip this section.

In this work, we will be using the theory of linear approximations in two major
ways.  First, they will be used to define Decoupled DNNs (\pref{sec:DDNNs})
where we take the activation function for a value channel layer to be a linear
approximation of the activation function for the activation channel. Second,
they will be used to investigate how a DDNN's output varies as the weights of a
particular layer in the value channel are varied. Our key result
(\pref{thm:LinearDDNN}) implies that, when modifying a single layer, the linear
approximation is \emph{exact}, allowing us to encode the repair problem as an
LP.

Our notation is closest to~\citet{loomis1968advanced}, although the definitions
are entirely equivalent to those often used in undergraduate-level calculus and
real analysis courses. We assume an intuitive familiarity with the concept of a
limit.

\subsection{The Scalar Case}
Suppose we have some scalar-valued function $f : \mathbb{R} \to \mathbb{R}$. A
central question of calculus is whether $f$ can be locally well-approximated by
a linear function around a point $x_0$. In other words, we want to find some
slope $s$ such that
$
    f(x_0 + \Delta) \approx f(x_0) + s\Delta
$
when $\Delta$ is close to zero. Note that the right-hand side is a linear
function of $\Delta$. This is not yet a particularly rigorous statement of what
properties we would like this linear approximation to have.  This is due to the
imprecise notion of what we mean by $\approx$ and `close to zero.' To make our
notion of approximation more precise, we will use the \emph{little-oh
notation,}~\citep{graham1989concrete} as defined below.  Intuitively, we say
that $f(x) = o(g(x))$ as $x \to x_0$ if $g(x)$ is ``meaningfully larger'' than
$f(x)$ close to $x_0$.

\begin{definition}
    \label{def:littleoh}
    For functions $f$, $g$ we say \emph{$f(x) = o(g(x))$ as $x$ approaches
    $x_0$} if
    $
        \lim_{x\to x_0} \frac{f(x)}{g(x)} = 0.
    $
\end{definition}

We will now rephrase our approximation requirement as:
$
    f(x_0 + \Delta) = f(x_0) + s\Delta + o(\Delta)
$
as $\Delta$ approaches 0. In particular, this notation means that
$
    f(x_0 + \Delta) - (f(x_0) + s\Delta) = o(\Delta).
$
Intuitively, the left hand side of this equation is an error term, i.e., how
far away our linear approximation is from the true value of the function. If
our linear approximation is reasonable, this error term should become
arbitrarily small as we get closer to $x_0$.

Notably, we do not just want the error term to approach zero, because then for
any continuous function we could simply use the horizontal line ($s = 0$) as
the linear approximation. For example, we generally do not think of $f(x) =
\abs{x}$ as being well-approximated by any line at $x = 0$. Hence, we need to
place some additional requirement on how fast the error decreases as $\Delta$
approaches zero. Because we are attempting to approximate the function with a
linear function, we expect the error to decrease faster than any other
linear-order function. This is captured by the $o(\Delta)$ on the right-hand
side.  The $o(\Delta)$ bound can also be motivated by noting that, if a
function is well-approximated around some point by a line, then that line
should be unique. In other words, if there were some other $t$ for
which
$
    f(x_0 + \Delta) \approx f(x_0) + t\Delta,
$
then we would like to ensure that $t = s$. Ensuring the error term is
$o(\Delta)$ ensures this fact, as subtracting the two gives us
$(s - t)\Delta = o(\Delta)$, i.e., $\frac{(s - t)\Delta}{\Delta} \to 0$, which
implies $s = t$. Duly motivated, and noting that this $s$ if it exists is
unique, we can formally define the scalar derivative:

\begin{definition}
    Suppose $f : \mathbb{R} \to \mathbb{R}$ and $x_0 \in \mathbb{R}$. We say
    that $f$ is \emph{differentiable at $x_0$} if there exists $s \in
    \mathbb{R}$ such that
    $
        f(x_0 + \Delta) = f(x_0) + s\Delta + o(\Delta)
    $
    as $\Delta$ approaches zero. In that case, we say $s$ is the
    \emph{derivative of $f$ at $x_0$} and write $f'(x_0) = s$ or
    $\frac{df}{dx}(x_0) = s$.
\end{definition}

In fact, expanding the definition of $o(\Delta)$ using~\pref{def:littleoh}, we
find that this is equivalent to the statement
$
    \lim_{\Delta\to 0} \frac{f(x_0 + \Delta) - (f(x_0) + s\Delta)}{\Delta} = 0
$, or equivalently
$\lim_{\Delta\to 0} \frac{f(x_0 + \Delta) - f(x_0)}{\Delta}$ $= s$,
the familiar calculus definition of a derivative.

\begin{example}
    Consider the function $f(x) = x^2$, which we will show is differentiable at
    any point $x_0$ with derivative $2x_0$. In particular, we must show
    $
        (x_0 + \Delta)^2 = x_0^2 + (2x_0)\Delta + o(\Delta).
    $
    Expanding the left-hand side gives us
    $
        x_0^2 + 2x_0\Delta + \Delta^2 = x_0^2 + 2x_0\Delta + o(\Delta).
    $
    Two of the terms on either side are entirely identical, hence this holds if
    and only if
    $
        \Delta^2 = o(\Delta),
    $
    and indeed we have
    $
        \lim_{\Delta \to 0} \frac{\Delta^2}{\Delta} = \lim_{\Delta \to 0} \Delta = 0.
    $
    In fact, more generally for any $c > 1$ we have $\Delta^c = o(\Delta)$ as $\Delta \to 0$.
    Thus, we have shown that $f$ is differentiable at every $x_0$ with
    derivative $2x_0$, as expected from calculus.
\end{example}

\subsection{The Vector Case}
We now assume $f$ is higher dimensional, i.e., $f : \mathbb{R}^n \to
\mathbb{R}^m$, and we ask the same question: is $f$ well-approximated by a
linear function around some $\vec{x_0} \in \mathbb{R}^n$? The
higher-dimensional analogue to a one-dimensional linear function is a
\emph{matrix multiplication} followed by a constant addition. In other words,
we would like to find some matrix $J$ such that
$
    f(\vec{x_0} + \vec{\Delta}) \approx f(\vec{x_0}) + J\vec{\Delta}
$
when $\vec{\Delta}$ is close to zero. Here again we run into two issues: what,
rigorously, does $\approx$ and close to zero mean? Using the insight from the
scalar case, we may attempt to rigorously define this as
$
    f(\vec{x_0} + \vec{\Delta}) = f(\vec{x_0}) + J\vec{\Delta} + o(\vec{\Delta})
$
as $\vec{\Delta}$ approaches zero. However, expanding the definition of
little-oh, we see this would require us to divide two vectors, which is not a
well-defined operation in general. Instead, we return to the intuition that we
want the error to get closer to zero significantly faster than $\vec{\Delta}$
does. Here, we need a notion of `close to zero' for vectors that can be
compared via division. This is given by the notion of a \emph{norm,} defined
below.

\begin{definition}
    A \emph{norm} on $\mathbb{R}^n$ is any function $\norm{\cdot} :
    \mathbb{R}^n \to \mathbb{R}$ that satisfies: (i) $\norm{\vec{v}} = 0$ only
    if $\vec{v} = 0$, (ii)  $\norm{c\vec{v}} = \abs{c}\norm{\vec{v}}$ for any
    $c \in \mathbb{R}$, and (iii) $\norm{\vec{v} + \vec{w}} \leq \norm{\vec{v}}
    + \norm{\vec{w}}$.
\end{definition}

Intuitively, $\norm{\vec{v}}$ gives us the distance of $\vec{v}$ from the
origin. Many different norms on $\mathbb{R}^n$ can be defined, however it can
be shown that they all result in equivalent notions of differentiability and
derivative. Two norms are particularly worth mentioning here (i) the $\ell_1$
norm defined $\norm{\vec{v}}_1 = \sum_{i=1}^n \abs{v_i}$, and (ii) the
$\ell_\infty$ norm defined $\norm{\vec{v}}_\infty = \max_i \abs{v_i}$. These
two norms are particularly useful to us because they can be encoded as
objectives in an LP.

Now that we have a notion of closeness to zero that can be divided, we can
formalize our notion of a differentiable function in higher dimensions:
\begin{definition}
    \label{def:Jacobian}
    Suppose $f : \mathbb{R}^n \to \mathbb{R}^m$ and $\vec{v_0} \in
    \mathbb{R}^n$. We say that $f$ is \emph{differentiable at $\vec{v_0}$} if
    there exists a matrix $J$ such that
    $
        f(\vec{v_0} + \vec{\Delta}) = f(\vec{v_0}) + J\vec{\Delta} + o(\norm{\vec{\Delta}})
    $
    as $\vec{\Delta}$ approaches zero, where $J\vec{\Delta}$ is a matrix-vector
    multiplication. In that case, we say $J$ is the \emph{Jacobian derivative
    (Jacobian) of $f$ at $\vec{v_0}$} and write $D_{\vec{v}} f(\vec{v_0}) = J$.
\end{definition}
Note in the above definition that the error term is a vector. Denoting it
$\vec{E}$, the limit corresponding to the little-oh definition is actually
asserting that
$
    \lim_{\vec{\Delta} \to \vec{0}} \norm{\vec{\Delta}}^{-1}\vec{E} = \vec{0},
$
i.e., approaching the zero \emph{vector}. It can be shown that this is
equivalent to the scalar-valued limit
$
    \lim_{\vec{\Delta} \to \vec{0}} \norm{\vec{\Delta}}^{-1}\norm{\vec{E}} = 0.
$

By a similar argument to the scalar case, it can be shown that this
$D_{\vec{v}} f(\vec{v_0})$, if it exists, is unique. In fact, it can also be
shown that the entries of $J$ are the \emph{partial derivatives} of each of the
$m$ scalar components of $f$ with respect to each of the $n$ scalar components
of $\vec{v}$. Jacobians can be computed automatically using standard automatic
differentiation packages like~\citet{paszke2017automatic}.

\subsection{A Vector Example}
\label{sec:JacobianExample}
Consider the DNN given by the function
$
    f(\vec{v}, W) = ReLU(W\vec{v})
$
where $W$ is a weight matrix with two rows and two columns. Note that we have
written our function $f$ as a function of \emph{both} $\vec{v}$ and $W$,
because, in this paper, we will be interested in how the function behaves as we
vary $W$. Now, fix
$
    \vec{v_0} \coloneqq \begin{bmatrix}1 \\ 2\end{bmatrix}
$
and
$
    W_0 \coloneqq
    \begin{bmatrix}
        1 & 2 \\
        3 & -4 \\
    \end{bmatrix}.
$

We want to know: is $f(\vec{v_0}, W)$, as a function of $W$, differentiable at
$W_0$? In fact, it is. To show this, we must produce some matrix $J$ that
forms the linear approximation
$
    f(\vec{v_0}, W_0 + \vec{\Delta}) = f(\vec{v_0}, W_0) + J\vec{\Delta} + o(\norm{\vec{\Delta}}).
$
The reader may now be concerned, as $\vec{\Delta}$ is a \emph{matrix}, not a vector.
However, the set of all $2\times 2$ (or generally, $n\times m$) matrices forms
a vector space $\mathbb{R}^{2\times 2}$ (generally $\mathbb{R}^{n\times m}$).
We think of $J$ as a matrix \emph{over the vector space of matrices}. This can
be understood using the fact that $\mathbb{R}^{2\times 2}$ is isomorphic
to $\mathbb{R}^4$ (generally $\mathbb{R}^{n\times m} \simeq \mathbb{R}^{nm}$).
Hence we can think of $W_0$ and $\vec{\Delta}$ as flattened versions of their
corresponding matrices, i.e.,
$
    \vec{W_0} = \begin{bmatrix}1& 2& 3& -4\end{bmatrix}^T
$
if we flatten in row-major order, and similarly for $\vec{\Delta}$. Then $J$ is a
standard matrix with 4 columns (one for every entry in $\vec{\Delta}$) and 2 rows
(one for every output of $f$).

Using the constructive characterization of the Jacobian alluded to above, we
can find that
${\footnotesize
    J \coloneqq
    \begin{bmatrix}
        1 & 2 & 0 & 0 \\
        0 & 0 & 0 & 0 \\
    \end{bmatrix}}
$
works. To verify this we consider
\[
    {\footnotesize
    \begin{aligned}
        ReLU\left(
        \begin{bmatrix}
            1 + \Delta_1 & 2 + \Delta_2 \\
            3 + \Delta_3 & -4 + \Delta_4 \\
        \end{bmatrix}
        \begin{bmatrix}1 \\ 2\end{bmatrix}
        \right)
        =
        &ReLU\left(
        \begin{bmatrix}
            1 & 2 \\
            3 & -4 \\
        \end{bmatrix}
        \begin{bmatrix}1 \\ 2\end{bmatrix}
        \right) \\
        &+
        \begin{bmatrix}
            1 & 2 & 0 & 0 \\
            0 & 0 & 0 & 0 \\
        \end{bmatrix}
        \begin{bmatrix} \Delta_1 \\ \Delta_2 \\ \Delta_3 \\ \Delta_4 \end{bmatrix}
            + o(\norm{\vec{\Delta}})
    \end{aligned}
        }
\]
which simplifies to
$
    ReLU\left(
    \begin{bmatrix}
        5 + \Delta_1 + 2\Delta_2 \\
        -5 + \Delta_3 + 2\Delta_4 \\
    \end{bmatrix}
    \right)
    =
    \begin{bmatrix}
        5 + \Delta_1 + 2\Delta_2 \\
        0 \\
    \end{bmatrix}
    + o(\norm{\vec{\Delta}}).
$

In fact, for small enough $\norm{\vec{\Delta}}$ there is \emph{no} error at all in
the approximation. Namely, when $\Delta_1 + 2\Delta_2 > -5$ and $\Delta_3 +
2\Delta_4 < 5$ we have that the first row inside the ReLU is positive and the
second row is negative, so the entire equation becomes
\[
    {
    \begin{bmatrix}
        5 + \Delta_1 + 2\Delta_2 \\
        0 \\
    \end{bmatrix}
    =
    \begin{bmatrix}
        5 + \Delta_1 + 2\Delta_2 \\
        0 \\
    \end{bmatrix}
    + o(\norm{\vec{\Delta}})}.
\]
for which the error is $0 = o(\norm{\vec{\Delta}})$ as desired. Therefore, we say
that $f$ is differentiable with respect to $W$ at $\vec{v_0}, W_0$ with
Jacobian
$
    D_W f(\vec{v_0}, \vec{W_0})
    =
    \begin{bmatrix}
        1 & 2 & 0 & 0 \\
        0 & 0 & 0 & 0 \\
    \end{bmatrix}.
$
Note that in this example the error term was zero for small enough
$\norm{\vec{\Delta}}$. This happens more generally only when the activation functions
used are piecewise-linear. However, other activation functions are still
differentiable and we can still compute Jacobians.

\section{Computing Jacobians on Vertices of Linear Regions}
\label{app:VertexJacobians}
In~\pref{alg:polytopewise} we have ignored a major subtlety in the reduction.
Namely, the vertex points which we will call~\pref{alg:pointwise} on are points
lying on the boundary between linear regions, and hence are precisely those
(almost-nowhere) points where the network is technically non-differentiable.

To make the algorithm correct, one must be particularly careful when performing
this Jacobian computation. In particular, we should associate with each vertex
$v$ in~\pref{alg:polytopewise} the associated linear region $R$ which it is
representing. Then, when we repair in~\pref{alg:pointwise}, we should compute
the Jacobian using the activation pattern shared by points interior to linear
region $R$. In other words, we repair the vertices of a linear region $R$ by
assuming that that linear region extends to its boundaries.
Geometrically, we can imagine the surface of a PWL DNN as a large number of
polytopes tiled together, each having different slopes (normal vectors). When
repairing the vertex of a particular linear region, we want to use the normal
vector corresponding to that linear region, not one of the other adjacent ones.
In particular, this means that the same point may appear multiple times in the
final point repair specification, each time it appears it is instructing to
repair it as if it belongs to a different linear region.

\section{Non-Differentiable Functions}
\label{app:NonDfbl}
Interestingly, DDNNs can be extended to non-differentiable functions.
\pref{thm:CoupledDDNN}~and~\pref{thm:LinearDDNN} only require that the
linearization agrees with the actual function at its center, not any other
property of \pref{def:Linearize}. Consequently, non-differentiable functions
can be ``linearized'' by taking an arbitrarily-sloped line centered at that
point.  In fact, this is what we do in the zero-probability cases where the
input to the ReLU is 0, we arbitrarily (although consistently) pick the
linearization to be the zero line. For polytope repair more care needs to be
taken when deciding how to handle such points, see~\pref{app:VertexJacobians}
for details.

As we will see, point repair does not rely on~\pref{thm:DDNNSameLinRegions},
hence~\pref{alg:pointwise} works even for non-differentiable activation
functions.  However, polytope repair relies on~\pref{thm:DDNNSameLinRegions}
hence on the fact that it is an actual linearization. However, this is somewhat
of a curiosity, as all common activation functions are differentiable
almost-everywhere in order to apply gradient descent.
\section{Evaluation Tables (Extended)}
\pref{tab:ImageNet2} gives extended results for the ImageNet \task{1}
experiments.

\begin{table*}[t]
    \footnotesize
    \caption{Summary of experimental results for \task{1}. Efficacy column
    gives the number of layers (out of 10) for which a satisfying repair could
    be found. $^{\ \ast}$Gurobi timed out on one of the layers; for the other
    two layers, Gurobi was able to prove that no repair exists.  These three
    cases are not included in the timing numbers for the last row.}
    \begin{tabular}{@{}lc@{\hskip .5cm}ccc@{\hskip .5cm}rrr@{\hskip .5cm}cr@{\hskip .5cm}cr}
        Points & Efficacy
        & \multicolumn{3}{c}{Drawdown (\%)}
        & \multicolumn{3}{c}{Time}
        & \multicolumn{2}{c}{FT[1]}
        & \multicolumn{2}{c}{FT[2]}
        \\
        \cmidrule(r){3-5} \cmidrule(r){6-8} \cmidrule(r){9-10} \cmidrule(r){11-12}
        &
        & Best & Worst & Fastest
        & Fastest & Slowest & Best Drawdown
        & Drawdown & Time
        & Drawdown & Time
        \\
        \midrule
        100 & 10 / 10
        & 3.6 & 39.0 & 39.0
        & 46.3s & 3m40.4s & 1m39.0s
        & 10.2 & 4m31.8s
        & 8.2 & 9m24.0s
        \\
        200 & 10 / 10
        & 1.1 & 40.8 & 31.4
        & 1m26.8s & 7m20.8s & 2m50.8s
        & 9.6 & 12m19.5s
        & 9.6 & 26m35.0s
        \\
        400 & 10 / 10
        & 5.1 & 51.4 & 45.0
        & 3m2.8s & 17m55.6s & 4m45.3s
        & 13.8 & 34m2.6s
        & 11.1 & 1h9m26.8s
        \\
        800 & 7$^\ast$ / 10
        & 5.3 & 58.8 & 58.8
        & 6m7.9s & 58m9.7s & 8m28.1s
        & 15.4 & 1h22m18.7s
        & 13.4 & 2h33m8.2s
    \end{tabular}
    \label{tab:ImageNet2}
\end{table*}

}{}

\end{document}